\documentclass{article}

\usepackage{PRIMEarxiv}

\usepackage[utf8]{inputenc} 
\usepackage[T1]{fontenc}    
\usepackage[colorlinks=true, urlcolor=blue]{hyperref}
\usepackage{url}            
\usepackage{booktabs}       
\usepackage{amsfonts}       
\usepackage{nicefrac}       
\usepackage{microtype}      
\usepackage{xcolor}         
\usepackage{lipsum}
\usepackage{fancyhdr}       
\usepackage{graphicx}       
\usepackage{amsmath}
\usepackage{definitions}
\usepackage{comment}
\usepackage{subcaption}
\usepackage{enumitem}
\usepackage{cleveref}
\usepackage{array}
\usepackage{algorithm}
\usepackage{algorithmic}

\newcommand{\gx}{g_{i}^x}
\newcommand{\gy}{g_{i}^y}
\newcommand{\highlight}[1]{\colorbox{blue!10}{#1}}

\title{Understanding and Improving Neural Active Learning on Heteroskedastic Distributions
\thanks{Correspondence to \href{mailto:savyakhosla08@gmail.com}{savyakhosla08@gmail.com}} 
}

\author{
  Savya Khosla \\
  University of Illinois Urbana-Champaign \\
  \And
  Chew Kin Whye \\
  National University of Singapore \\
  \And
  Jordan T. Ash \\
  Microsoft Research NYC \\
  \And
  Cyril Zhang \\
  Microsoft Research NYC \\
  \And
  Kenji Kawaguchi \\
  National University of Singapore \\
  \And
  Alex Lamb \\
  Microsoft Research NYC \\
}

\begin{document}
\maketitle

\newtheorem{thm}{Theorem}
\newtheorem{lem}[thm]{Lemma}
\newtheorem{cor}[thm]{Corollary}
\newtheorem{rem}[thm]{Remark}
\newtheorem{remark}[thm]{Remark}
\newtheorem{conj}[thm]{Conjecture}
\newtheorem{proposition}[thm]{Proposition}

\begin{abstract}
Models that can actively seek out the best quality training data hold the promise of more accurate, adaptable, and efficient machine learning. Active learning techniques often tend to prefer examples that are the most difficult to classify. While this works well on homogeneous datasets, we find that it can lead to catastrophic failures when performed on multiple distributions with different degrees of label noise or heteroskedasticity. These active learning algorithms strongly prefer to draw from the distribution with more noise, even if their examples have no informative structure (such as solid color images with random labels). To this end, we demonstrate the catastrophic failure of these active learning algorithms on heteroskedastic distributions and propose a fine-tuning-based approach to mitigate these failures. Further, we propose a new algorithm that incorporates a model difference scoring function for each data point to filter out the noisy examples and sample clean examples that maximize accuracy, outperforming the existing active learning techniques on the heteroskedastic datasets. We hope these observations and techniques are immediately helpful to practitioners and can help to challenge common assumptions in the design of active learning algorithms. Our code is available at \href{https://github.com/savya08/Active-Learning-on-Heteroskedastic-Distributions}{this URL}.
\end{abstract}
\section{Introduction}
In an active learning setup, a model has access to a pool of labeled and unlabeled data. After training on the available labeled data, a selection rule is applied to identify a batch of $k$ unlabeled examples to be labeled and integrated into the training set before repeating the process. Under this paradigm, data is considered to be abundant, but label acquisition is costly. An active learning algorithm aims to identify unlabeled examples that, once labeled and used to fit model parameters, will elicit the most performant hypothesis possible given a fixed labeling budget. To fulfill this objective, a selection criteria generally follows two heuristics: (1) select diverse examples and (2) select examples where the model has a high degree of uncertainty. 

The presence of noise is an unavoidable problem that corrupts real-world datasets \cite{Wang95quality} and is detrimental to the performance of classifiers directly trained on them. Active learning can seek to combat this problem through a robust data-selection pipeline that effectively filters out the noisy data and select the most informant examples for machine learning, allowing us to efficiently leverage the abundance of data we have at our disposal. 

In our study, we explore the performance of the active learning algorithms on \emph{heteroskedastic distributions}, where the training data consists of a mixture of distinct distributions with different degrees of noise. One use case of active learning on heteroskedastic distributions is in reinforcement learning, where the agent gathers its own training data through its actions and obtains the feedback from the environment. Often, the environment has heteroskedastic noise. For example, certain actions like "opening a box with a question mark symbol" have random outputs, whereas other actions like "moving left/right" have deterministic outputs. Another possible use case of our study could be for training LLMs using techniques like self-instruct \cite{self-instruct}. Recent works \cite{self-instruct, alpaca, llama-adapter} have shown that curating a dataset using large language models could help generate more diverse training examples. However, this method of curating data is inherently noisy, and therefore the data pipeline uses a critical filtering step to remove redundant and non-informative examples. We believe that the findings in this paper could also help reduce the dependency on clean data allowing practitioners to train models on larger sets.

We find that preferring examples with high uncertainty often works well on homogeneous datasets but can lead to catastrophic failure when training on heteroskedastic distributions. Uncertainty-based active learning algorithms typically rely on notions of model improvement that are unable to disambiguate aleatoric uncertainty from epistemic uncertainty, thereby over-selecting examples for which the model is unconfident but which are unlikely to improve the current hypothesis. We produce a generalization bound that explains why this phenomenon occurs, which seems to superficially contradict previous theory \cite{kawaguchi2020ordered} that showed training only on examples with high loss could generalize as well as training on randomly selected examples.

\begin{figure*}
    \centering
    \begin{subfigure}[b]{0.5\textwidth}
        \centering
        \includegraphics[width=0.5\linewidth]{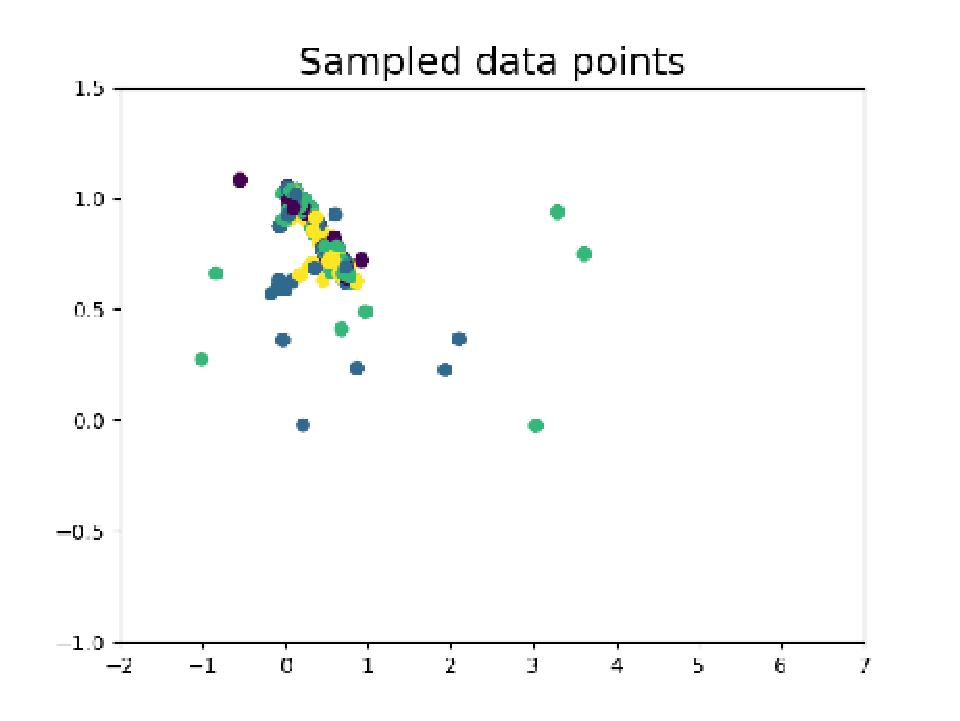}%
        \includegraphics[width=0.5\linewidth]{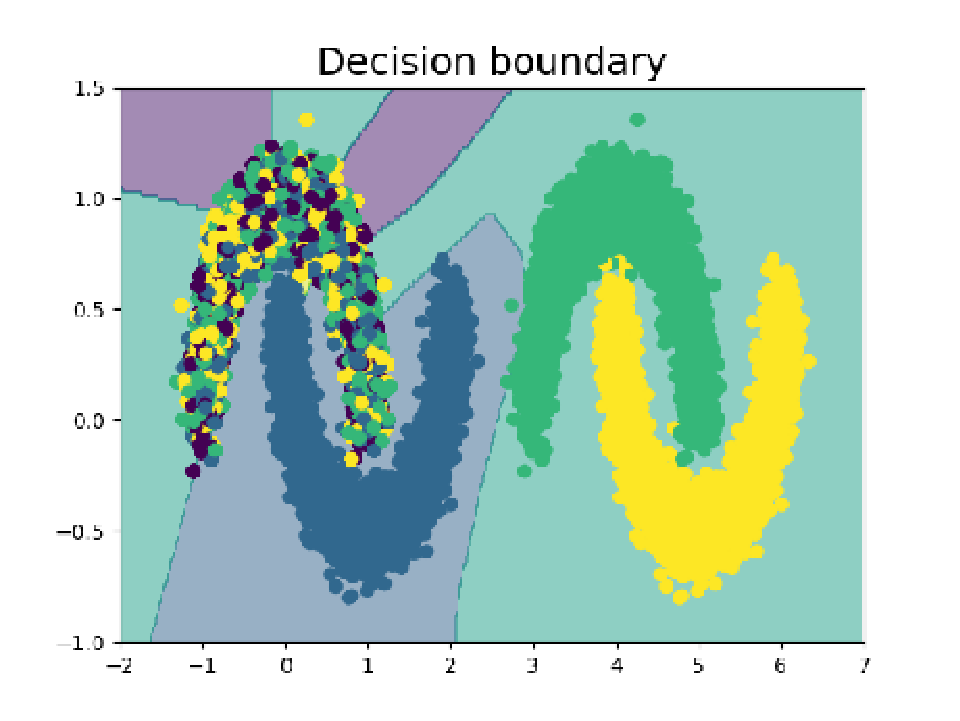}
        \caption{Least-Confidence Sampling: Test accuracy = 53.67\%}
    \end{subfigure}%
    \begin{subfigure}[b]{0.475\textwidth}
        \centering
        \includegraphics[width=0.5\linewidth]{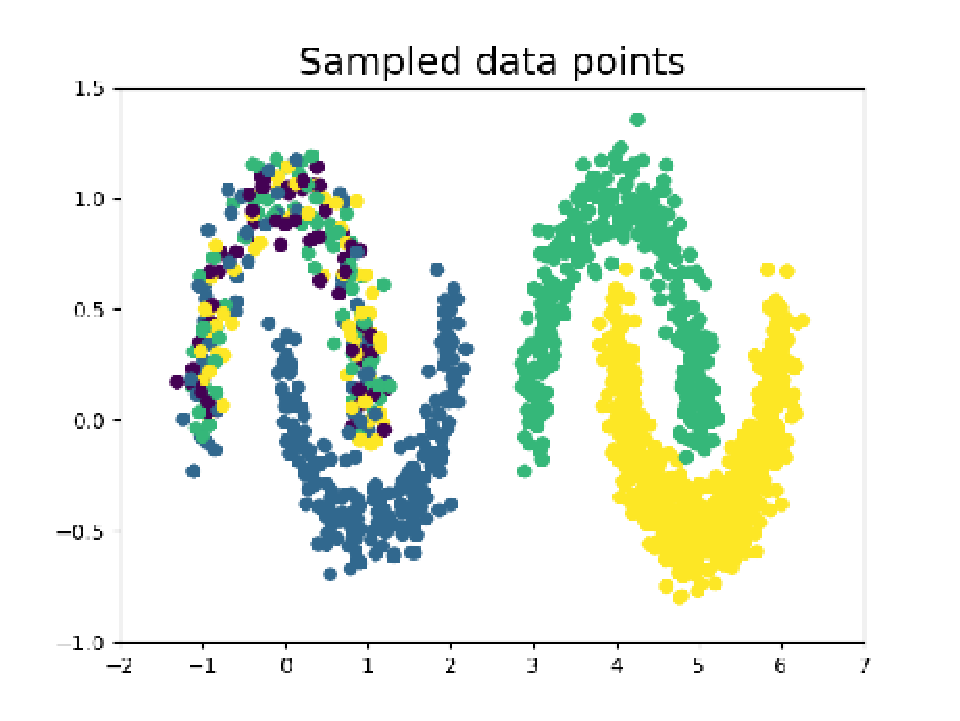}%
        \includegraphics[width=0.5\linewidth]{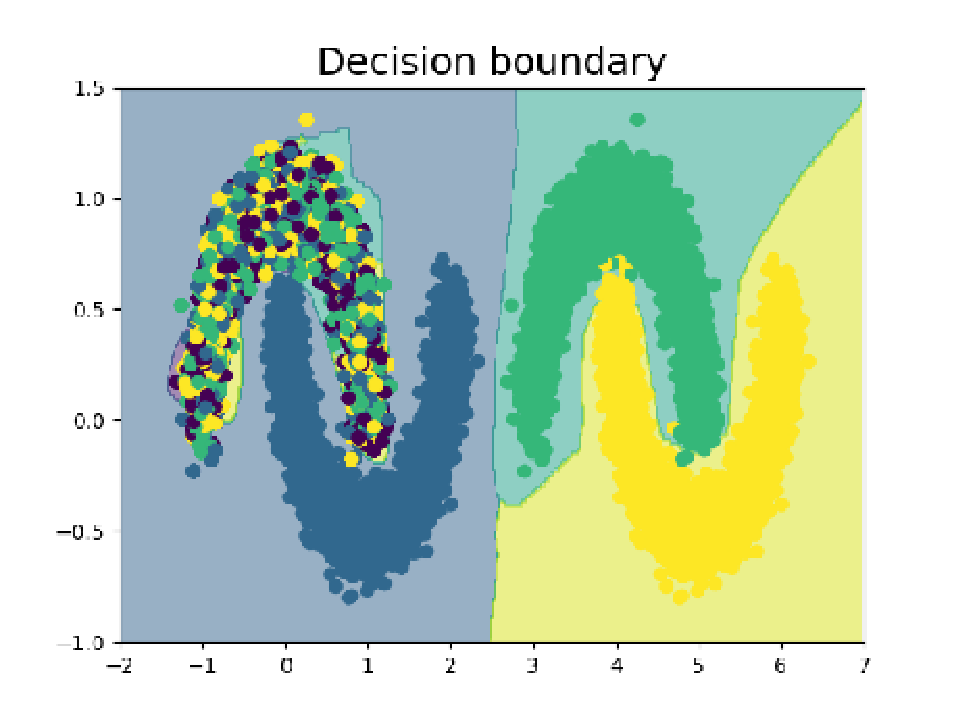}
        \caption{\textsc{LHD} Sampling: Test accuracy = 84.68\%}
    \end{subfigure}
    \caption{We construct \textit{Four-Moons Dataset}, a toy dataset with heteroskedastic noise. In this dataset, we have four classes - blue, green, yellow, and purple. The data points belonging to the purple class (top-left moon) are assigned uniformly random labels, while the points in other moons has no noise. The least-confidence sampling (uncertainty-based algorithm) selects examples almost exclusively from the noisy class (top-left moon), resulting in a poor decision boundary. The proposed LHD algorithm, on the other hand, promotes selection of clean data points and almost perfectly solves the classification problem.}
    \label{fig:intro-teaser}
\end{figure*}

Further, we show that this inefficiency in the data-selection process can be mitigated in three ways.

1. \textbf{Favoring diversity over uncertainty.} As we demonstrate empirically in this work, active learning algorithms that promote the selection of a diverse set of examples can efficiently filter out the data points with heteroskedastic noise since the noisy examples come from the same distribution and therefore have similar feature representations.

2. \textbf{Leveraging high confidence examples in the unlabeled pool of data.} We show that the examples in the unlabeled pool for which the model is highly confident can promote better feature learning, thereby improving the performance of active learning algorithms even in the presence of heteroskedastic noise in the dataset.

3. \textbf{Encouraging the selection of examples for which model's representations change over training iteration.} We show that the model's representation for examples with heteroskedastic noise converges quickly to a suboptimal solution. This can be used as a helpful signal to filter out these noisy examples.

The change in model's representation is measured using the difference between a conventionally-trained model and an exponential moving average (EMA) of its iterates. For the noisy examples in the dataset, both the conventionally-trained model and the EMA model converges quickly to the suboptimal solution. Consequently, the EMA difference is nearly zero for these examples, and thus, it can be used as a helpful signal to filter out the noise from the dataset. Further, as we show later in the paper, the EMA difference is maximum for examples that are difficult to classify (but not noisy), thereby promoting the selection of challenging yet clean data.

The toy example with heteroskedastic noise in Figure~\ref{fig:intro-teaser} shows how the least-confidence sampling (uncertainty-based algorithm) selects examples almost exclusively from the noisy class (top-left moon), resulting in an extremely poor decision boundary. The Coreset algorithm (diversity-based algorithm) and the proposed LHD algorithm promotes selection of clean data points and almost perfectly solves the classification problem.



The main contributions of this work are:

\begin{enumerate}
    \item We study the performance of active learning algorithms on heteroskedastic distributions and show that algorithms that exclusively prefer uncertainty can catastrophically fail in the presence of heteroskedasticity.
    \item We produce a generalization bound that explains why training only on low confidence examples can lead to poor performance in the presence of heteroskedastic noise.
    \item We explore a fine-tuning-based approach that helps improve the performance of all algorithms in the presence of heteroskedasticity.
    \item We propose an algorithm, hereafter referred to as LHD, that performs comparably to the existing state-of-the-art algorithms in the general setup and, when coupled with fine-tuning, outperforms all algorithms by a significant margin.
\end{enumerate}

\section{Related Work}
\textbf{Neural active learning.} Active learning is an extremely well-researched area, with the richest theory developed for the convex setting \cite{D11,chaudhuri2015convergence,chaudhuri2017active}. More recently, however, there have been several attempts to tractably generalize active learning to the deep regime. Such approaches can be thought of as identifying batches of samples that cater more to either the model's predictive uncertainty or to the diversity of the selection.

In the former approach, a batch of points is selected in order of the model's uncertainty about their label. Many of these methods query samples that are nearest the decision boundary, an approach that's theoretically well understood in the linear regime when the batch size is $1$ \cite{tur2005combining}. Some deep learning-specific approaches have also been developed, including using the variance of dropout samples to quantify uncertainty \cite{gal2017deep}, and adversarial examples have been used to approximate the distance between an unlabeled sample and the decision boundary. In the deep setting, however, where models are typically retrained from scratch after every round of selection, larger batch size is usually necessary for efficiency purposes.

For large acquisition batch sizes, algorithms that cater to diversity are usually more effective. In deep learning, several methods take the representation obtained at the penultimate layer of the network and aim to identify a batch of samples that might summarize this space well \cite{sener2018active,geifman2017deep,gissin2019discriminative}. Other methods promote diversity by minimizing an upper bound on some notion of the model's loss on unseen data \cite{wang2015querying,pmlr-v28-chen13b,wei2015submodularity,batchbald}. This approach has also been taken to a trade-off between diversity and uncertainty in deep active learning \cite{ash2021gone,ash2019deep}.

\textbf{Data poisoning, distributional robustness, and label noise.} A related body of work seeks to obtain models and training procedures that are robust against \emph{worst-case} perturbations to the data distribution. For recent treatments of this topic and further references, see \cite{steinhardt2017certified,sagawa2019distributionally}. A few recent works have considered data poisoning in the active learning setting \cite{lin2021active,vicarte2021double}, with defenses focusing on modifying the setting rather than the algorithm. Further, some existing works in active learning regime \cite{similar2021al,contrastive2021al} consider the presence of label noise, out-of-distribution examples, and redundancy in the dataset. Our work, however, considers the setting wherein the system suffers from low-quality labels (e.g., in medical diagnosis, where the labelers are not always adept at assigning the correct label to the example queried by the algorithm and might end up incorrectly assigning out-of-distribution examples to one of the classes in the label space).

\textbf{Heteroskedasticity in machine learning.}  The issues of class imbalance and heteroskedasticity are of interest in the supervised learning setting \cite{shu2019meta,cao2020heteroskedastic}, in which various methods have been proposed to make training more robust to these distributions. Our work seeks to initiate the study of the orthogonal (but analogous) issue in the sample \emph{selection} regime. Similar to our work, \cite{ANTOS20102712} considers a theoretical regression problem that explores active learning in heteroskedastic noise. However, contrary to our setup, it \emph{leverages} heteroskedasticity in the pool of labeled data to sample more observations from the parts of the input space with large variance.

\textbf{Semi-supervised active learning.} Recent advances in semi-supervised learning (SSL) have demonstrated the potential of using unlabeled data for active learning. For instance, \cite{Zhu2003CombiningAL} combines SSL and AL using a Gaussian random field model.  \cite{consistencysslal} proposes to human label the unlabelled example for which the different augmented views result in inconsistent SSL model's predictions because such behavior indicates that the model cannot successfully distill helpful information from that unlabelled example. Similarly, \cite{bilevelopt} proposes to use a model trained using SSL to select a batch of unlabeled examples that best summarizes a pool of data pseudo-labeled by the model itself. \cite{ssl2019asr,Rhee2017ActiveAS} leverage active learning and semi-supervised learning in succession to show incremental improvements in speech recognition and object detection, respectively. \cite{Sinha2019vaal} learns the sampling criteria by setting up a mini-max game between a variational auto-encoder that generates latent representations for labeled and unlabeled data and an adversarial network that tries to discriminate between these representations. In this work, we experiment with a very simple SSL setup to see its effectiveness when performing active learning in the presence of heteroskedastic noise.

\begin{figure*}[!htbp]
    \centering
    \includegraphics[width=0.9\linewidth]{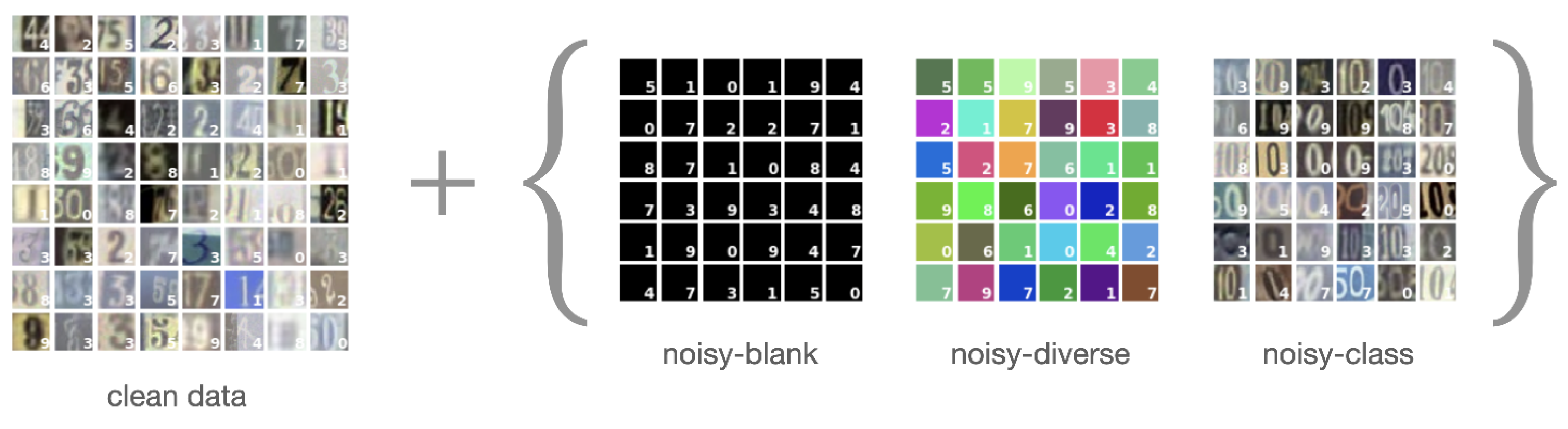}
    \caption{The heteroskedastic distributions proposed in this paper. The SVHN dataset is corrupted with randomly labeled examples, with (1) black images, (2) diverse coloured images, and (3) images from one class.}
    \label{fig:svhn-noise}
\end{figure*}

\section{Heteroskedastic Benchmarks for Neural Active Learning}
\label{heteroskedastic-data}
We introduce three benchmarks for active learning on heteroskedastic distributions. In all cases, we introduce an additional set of $N$ examples with purely random labels to the original clean data. $K$ is the number of unique noisy examples, since some of the examples are repeated.

The model is not given information on which samples are noisy/clean, but it is reasonably predictable from the example's features since the noisy datapoints are from the same distribution. This distinguishes our benchmarks from IID label noise, which is not predictable based on the example's features. These constructions are described below and summarized in Figure~\ref{fig:svhn-noise}.

\textbf{Noisy-Blank}: We introduce $N$ examples that are all solid black ($K = 1$) and have a random label $y \sim U(1,n_y)$, where $n_y$ refers to the number of classes. 

\textbf{Noisy-Diverse}: We increase the difficulty by introducing $K=100$ different types of examples, where each type is a random solid color and has a label randomly drawn from three label choices that are unique to that color. $N$ such noisy examples are introduced to the dataset. This benchmark is designed to make the heteroskedastic distribution more diverse while still keeping the noisy examples simple.

\textbf{Noisy-Class}: In our most challenging setting, we take $K$ examples from a particular class (say, $y=1$) and assign these examples uniformly random labels $y \sim U(1,n_y)$. We then randomly repeat these examples to give $N$ noisy examples. In this case, the randomly labeled examples are challenging but still possible to identify.  

We designed these benchmarks to easily evaluate the performance of existing algorithms. Despite being highly simplistic, these benchmarks do delineate a shortcoming of the existing active learning algorithms, and we believe that practitioners should make their active learning pipelines robust to such noisy adversaries. In this regard, we can draw an analogy with the domain of adversarial learning – while the adversary will not have access to the model weights and gradients in most realistic setups, practitioners want their pipelines to be robust to white-box adversarial attacks like Projected Gradient Descent. 

Future work can be done to generate more realistic datasets with heteroskedastic noise.
\section{Method}
In this section, we describe (1) the baseline active learning algorithms with which we experiment, (2) \textsc{LHD}, an active learning algorithm that leverages EMA difference to sample examples for which the model's representation changes over training iterations, (3) the fine-tuning technique used to improve the performance of the baseline algorithms. and (4) the combination of LHD with fine-tuning.

\subsection{Baselines}
We review some prominent neural active learning algorithms, which act as baselines in our study.

\textbf{Random sampling ($\mathsf{RAND}$).}
Unconditional random sampling from the unlabeled pool of data.

\textbf{Least confidence sampling ($\mathsf{CONF}$).} 
Confidence sampling selects the $k$ unlabeled points for which the most likely label has the smallest probability mass \cite{wang2014new}: 
\begin{center}
    $x^*_{\mathsf{CONF}} = \arg\min_{x} P_{\theta}(\hat{y} | x)$
\end{center}
Here, $P_{\theta}(\hat{y} | x)$ is the probability of the predicted (most likely) class $\hat{y}$ given the input $x$ and model parameters $\theta$, and $x^*_{\mathsf{CONF}}$ is the selected batch of data points.

\textbf{Margin sampling ($\mathsf{MARG}$).} 
Margin sampling selects the $k$ points for which the difference in probability mass in the two most likely labels is smallest \cite{roth2006margin}: 
\begin{center}
    $x^*_{\mathsf{MARG}} = \arg\min_{x} P_{\theta}(\hat{y}_1 | x) - P_{\theta}(\hat{y}_2 | x)$
\end{center}
where $\hat{y}_1$ and $\hat{y}_2$ are the first and second most probable classes, respectively.

\textbf{Bayesian Active Learning by Disagreements ($\mathsf{BALD}$).}
Here, the objective is to select $k$ points that maximize the decrease in expected posterior entropy \cite{bald2011}: 
\begin{center}
    $x^*_{\mathsf{BALD}} = \arg\max_{x} H[\theta|D] - \mathbb{E}_{y \sim p(y|x,D)}[H[\theta|y,x,D]]$
\end{center}
where $H[\cdot]$ represents entropy, $\theta$ represents model parameters, and $D$ represents the datatset.

\textbf{Coreset sampling ($\mathsf{CORESET}$).}
The Coreset algorithm is a diversity-based approach that aims to select a batch of representative points, as measured in penultimate layer space of the current state of the model \cite{sener2018active}.  We refer to the function for computing this penultimate layer as $h(x)$. It proceeds in these steps on each acquisition round:

(1) Given a set of existing selected unlabeled examples and labeled examples $x^*_{\mathsf{CORESET}}$ and a set of indices of these selected examples $s$.

(2) Select an example with the greatest distance to its nearest neighbor in the hidden space 
\begin{center}
    $u=\arg\max_{i \in [n] \setminus \mathbf{s}} \min_{j \in \mathbf{s}} \Delta(h(\mathbf{x}_i), h(\mathbf{x}_j))$
\end{center}

(3) Set $s = s \cup \{u\}$ and $x^*_{\mathsf{CORESET}} = x^*_{\mathsf{CORESET}} \cup \{X_u\}$.

(4) Repeat this in an active learning round until we reach the acquisition batch size.

\textbf{Batch Active learning by Diverse Gradient Embeddings ($\mathsf{BADGE}$).}
BADGE is a hybridized approach, meant to strike a balance between uncertainty and diversity. The algorithm represents data in a hallucinated gradient space before performing diverse selection using the k-means$++$ seeding algorithm \cite{ash2019deep}. It proceeds with these steps on each acquisition round:

(1) Compute hypothetical labels $\hat{y}(x) = h_{\theta_t}(x)$ for all unlabeled examples.

(2) Compute gradient embedding for each unlabeled example 
\begin{center}
    $g_x = \frac{\partial}{\partial \theta_\text{out}} \ell_{}(f(x;\theta),\hat{y}(x)) \vert_{\theta = \theta_t}$
\end{center}
where $\theta_\text{out}$ refers to the parameters of the output layer.

(3) Use $k$-means$++$ over the gradient embedding vectors $g_x$ over all unlabeled examples to select a batch of examples $x^*_{\mathsf{BADGE}}$.

\subsection{\textsc{LHD}: Increasing Sampling Where Representations Change Across Training Iterations}
\label{lhd}

For \emph{noisy} examples, conflicting gradients result in the model converging quickly to a suboptimal solution and undergoing little change throughout the training. For the \emph{clean} examples, the model converges slowly to an optimal solution, undergoing changes throughout the training (Figures \ref{fig:noisy-vs-clean-loss} and \ref{fig:conflicting-gradients} in the Appendix substantiate this claim). Therefore, by encouraging the selection process in active learning to select the examples for which the model converges slowly, we can maximize the sampling of \emph{clean} examples, improving the performance in the heteroskedastic setting.

To measure the convergence rate, in addition to the main model $F_{\theta}$, we introduce an exponential moving average (EMA) of the model $F_{\beta}$ in the training pipeline. The EMA model has the same architecture as the main model but uses a different set of parameters $\beta$, which are exponentially moving averages of $\theta$. That is, at epoch $t$, $\beta_{t+1} \leftarrow \alpha \cdot \beta_t + (1 - \alpha) \cdot \theta_t$, for some choice of decay parameter $\alpha$.  

The convergence rate for a training example is captured as the state difference between the main and the EMA model, which is measured in two ways: 

(1) \textbf{Loss Difference $\Delta l$}: The absolute difference between the loss values of an example from the EMA model and the main model: 
\begin{center}
    $\Delta l = \mid l_\text{ema} - l_\text{main} \mid $
\end{center}

For the unlabeled examples, we assume the prediction of the EMA model as the ground truth for loss computation. 

(2).  \textbf{Hidden State Difference $\Delta \mathbf{h}$}: The difference between the hidden feature representation from the penultimate layer of the EMA model and the main model: 
\begin{center}
    $\Delta \mathbf{h} = \mathbf{h_\text{ema}} - \mathbf{h_\text{main}}$.  
\end{center}


$\Delta l$ will be low for \emph{noisy} examples because both $l_\text{main}$ and $l_\text{ema}$ are high throughout the training. Similarly, $\Delta l$ will be low for \emph{simple-clean} examples because both $l_\text{main}$ and $l_\text{ema}$ are low for most part of the training. For \emph{difficult-clean} examples, however, the main model converges slowly, which leads to an even slower convergence of the EMA model. As a result, at some point in the training, we get a low $l_\text{main}$ but a high $l_\text{ema}$ for \emph{difficult-clean} examples, resulting in a high $\Delta l$. Using a similar line of reasoning, we can infer that $\Delta \mathbf{h}$ will have a smaller magnitude for the \emph{noisy} examples and \emph{simple-clean} examples, and a larger magnitude for \emph{difficult-clean} examples. Further, examples that are similar to one another will have similar $\Delta \mathbf{h}$, and a diversity-based sampling technique that operates on $\Delta \mathbf{h}$ will promote sampling of a diverse batch of examples.

We use $\Delta l$ and $\Delta \mathbf{h}$ to obtain the final \textbf{State Difference $\mathbf{lh}$} as:

\begin{center}
    $\mathbf{lh} = \Delta l \cdot \Delta \mathbf{h}$.
\end{center}

A training example with a small $||\mathbf{lh}||_2$ has a small state difference between the main and EMA model, meaning the example had converged very fast early on in the training, indicating that it is probably a \emph{noisy} example. (This can be seen from Figure~\ref{fig:wo-ft} in the Appendix).

Algorithm \ref{alg:lhd} describes \textsc{LHD} in detail. The state difference, $\mathbf{lh}$, is computed for all unlabeled examples. Then, we use the $k$-means$++$ seeding algorithm \cite{Arthur07kmeans} over all the $\mathbf{lh}$ embeddings, which selects a batch of diverse examples that have a high magnitude of $\mathbf{lh}$.

\begin{algorithm}
    \caption{LHD: \textbf{L}oss and \textbf{H}idden state \textbf{D}ifference sampling}
    \label{alg:lhd}
    \begin{algorithmic}[1]
        \REQUIRE Main model $F_\theta$, EMA model $F_\beta$, unlabeled pool of examples $U$, initial number of examples $M$, number of active learning iterations $T$, decay parameter $\alpha$, cross-entropy loss function $\text{CE}(\cdot)$, one-hot function $\text{OH}(\cdot)$.\
        \STATE Labeled dataset $S$ $\leftarrow$ $M$ examples drawn uniformly at random from $U$ together with their labels $y$.\
        \STATE Train an initial main model $F_{\theta_1}$ on $S$ while updating the initial EMA model $F_{\beta_1}$ using $\beta_{1} \leftarrow \alpha \cdot \beta_1 + (1 - \alpha) \cdot \theta_1$ in each training iteration
        \FOR{$t=1,2,...,T$} 
            \STATE Optionally fine-tune $F_{\theta_t}$ and $F_{\beta_t}$ using Algorithm \ref{alg:finetune}
            \FOR{all examples $x \in U \setminus S$}
                \STATE $y_{\text{pseudo}_x} = \textsc{OH}({F_{\beta_t}(x)})$
                \STATE $l_{\text{ema}_x} = \text{CE}(F_{\beta_t}(x), y_{\text{pseudo}_x})$ and $\mathbf{h_{\text{ema}}}_x = F^{\Hat{}}_{\beta_t}(x)$, where $F^{\Hat{}}_{\beta_t}(\cdot)$ represents a function to extract penultimate layer
                \STATE $l_{\text{main}_x} = \text{CE}(F_{\theta_t}(x), y_{\text{pseudo}_x})$ and $\mathbf{h_{\text{main}}}_x = F^{\Hat{}}_{\theta_t}(x)$, where $F^{\Hat{}}_{\theta}(\cdot)$ represents a function to extract penultimate layer
                \STATE $\Delta l_x = \mid l_{\text{ema}_x} - l_{\text{main}_x} \mid$ and $\Delta \mathbf{h}_x = \mathbf{h_{\text{ema}}}_x - \mathbf{h_{\text{main}}}_x$
                \STATE $\mathbf{lh}_x = \Delta l_x \cdot \Delta \mathbf{h}_x$
            \ENDFOR
            \STATE $S_t$ $\leftarrow$ a subset of $U \setminus S$ using the $k$-means++ seeding algorithm on $\{\mathbf{lh}_x : x \in U \setminus S\}$ and query their labels
            \STATE $S \leftarrow S \cup S_t$
            \STATE Train $F_{\theta_{t+1}}$ on $S$ by minimizing $E_S[\text{CE}(F_{\theta_t}(x), y)]$ and in each training iteration, update $F_{\beta_{t+1}}$ using $\beta_{t+1} \leftarrow \alpha \cdot \beta_t + (1 - \alpha) \cdot \theta_t$
        \ENDFOR
    \end{algorithmic}
\end{algorithm}

\begin{algorithm}
    \caption{Fine-tuning algorithm for improved example selection}
    \label{alg:finetune}
    \begin{algorithmic}[1]
        \REQUIRE Main model $F_\theta$, EMA model $F_\beta$, unlabeled pool of examples $U$, labeled pool of examples $S$, number of fine-tuning iterations $T$, decay parameter $\alpha$, confidence threshold $\gamma$, data augmentation function $\text{A}(\cdot)$, cross-entropy loss function $\text{CE}(\cdot)$, one-hot function $\text{OH}(\cdot)$.\
        \STATE Initialize $V \leftarrow \{\}$ a set of example selected for fine-tuning
        \STATE Initialize $y \leftarrow \{\}$ pseudo-labels for the fine-tuning examples
        \FOR{all examples $x \in U \setminus S$}
            \STATE Compute confidence $c_x = \max({F_{\theta_t}(x)})$
            \IF{$c_x > \gamma$}
                \STATE $y \leftarrow y \cup \{\textsc{OH}({F_{\beta_t}(x)})\}$
                \STATE $V \leftarrow V \cup \{\text{A}(x)\}$ 
            \ENDIF
        \ENDFOR
        \FOR{$t=1,2,...,T$}
            \STATE Train model $F_{\theta_{t+1}}$ on $V$ by minimizing $E_V[\text{CE}(F_{\theta_t}(x), y)]$ and in each training iteration, update $F_{\beta_{t+1}}$ using $\beta_{t+1} \leftarrow \alpha \cdot \beta_t + (1 - \alpha) \cdot \theta_t$
        \ENDFOR
    \end{algorithmic}
\end{algorithm}

\subsection{Fine-tuning: Using Unlabeled Data for Tackling Heteroskedasticity}
Traditionally, active learning algorithms train models on a small set of labeled data and use a large pool of unlabeled data only for sampling informative examples. However, we conjecture that the unlabeled pool can be efficiently leveraged to improve the performance of active learning algorithms even in the presence of heteroskedasticity. To this end, we experiment with an extremely simple semi-supervised learning technique (similar to \cite{sohn2020fixmatch}) to aid active learning.

After training the model on the labeled data points, we sample a batch of examples from the unlabeled pool for which the model is highly confident. Using the predicted labels for these examples as the ground truth, we fine-tune the model on strongly augmented versions of these confident examples. Since the model is inherently less confident in the noisy examples, most of the examples used for fine-tuning are clean. This way, we leverage the information in already well-classified clean examples, and the model learns more discriminative feature representations. Algorithm \ref{alg:finetune} outlines the fine-tuning procedure.

Since the fine-tuning technique allows us to exploit the information in the unlabeled data pool, all active learning methods benefits from the use of fine-tuning. Additionally, since the fine-tuning technique improves the quality of the representations, active learning algorithms that rely on these representations for the data selection obtains a more substantial benefit, especially in the earlier rounds where labeled data is scarce.


\subsection{LHD with Fine-tuning}
Even though the LHD method is able to filter out the noisy examples, it can still struggle to differentiate between the \emph{simple-clean} examples - examples from the original/clean data that are very easy to classify, and the \emph{difficult-clean} examples - examples from the original/clean data that are challenging for the model to classify.

When we add the fine-tuning method on top of LHD, the average $||\mathbf{lh}||_2$ for the \emph{difficult-clean} examples becomes much higher than the $||\mathbf{lh}||_2$ for \emph{simple-clean} examples. This can be seen from Figure~\ref{fig:w-ft} in the Appendix.

Essentially, the fine-tuning method trains on the examples that the model is confident in, which are mostly the \emph{simple-clean} examples. Therefore, the model will converge quickly to an optimal solution for the \emph{simple-clean} examples, undergoing little change later in training. On the other hand, since the model is not fine-tuned on the \emph{difficult-clean} examples, it takes a longer time to learn the correct solution, converging slowly to an optimal solution and undergoing changes throughout the training.

So, by encouraging the selection of examples for which the model converges slowly (the idea behind LHD) and fine-tuning the model on confident examples (the idea behind fine-tuning), we can maximize the sampling of \emph{difficult-clean} examples, thereby improving the performance of our active learning algorithm in heteroskedastic settings.

Empirically analyzing $\Delta l$ in this setting further substantiates the abovementioned claim. For \emph{noisy} examples, both $l_\text{main}$ and $l_\text{ema}$ are high throughout the training, resulting in a small $\Delta l$. For \emph{simple-clean} examples, both $l_\text{main}$ and $l_\text{ema}$ are low throughout the training, resulting in a small $\Delta l$. For \emph{difficult-clean} examples, since the loss decreases slowly, $l_\text{main}$ will be low while $l_\text{ema}$ is high, resulting in a large $\Delta l$. On similar lines, $||\Delta \mathbf{h}||_2$ has a larger value for \emph{difficult-clean} examples and a smaller value for \emph{simple-clean} and \emph{noisy} examples.

\begin{table}
\centering
\caption{Classification accuracy on CIFAR-10 with a Resnet model after 10 rounds of active learning. (\emph{+FT} represents Fine-tuning)}
\label{tab:cifar-all}
\resizebox{0.6\linewidth}{!}{
\begin{tabular}{lcccc}

\toprule
\textbf{Method} & \textbf{\shortstack[l]{Clean}} & \textbf{\shortstack[l]{Noisy-Blank}} & \textbf{\shortstack[l]{Noisy-Diverse}} & \textbf{\shortstack[l]{Noisy-Class}} \\
\midrule
$\text{RAND}$ & $41.09 \pm 0.43$ & $38.27 \pm 0.27$ & $32.70 \pm 0.18$ & $32.11 \pm 0.42$ \\ 
$\text{CONF}$ & $38.27 \pm 1.48$ & $32.88 \pm 0.25$ & $39.53 \pm 0.08$ & $27.04 \pm 0.26$ \\ 
$\text{MARG}$ & $37.86 \pm 4.33$ & $39.16 \pm 0.83$ & $26.34 \pm 3.25$ & $30.40 \pm 2.46$ \\ 
$\text{BALD}$ & $43.31 \pm 1.34$ & $46.90 \pm 0.39$ & $35.64 \pm 0.32$ & $31.65 \pm 0.30$ \\ 
$\text{CORESET}$ & $41.81 \pm 1.17$ & $47.19 \pm 2.03$ & \highlight{$47.33 \pm 3.31$} & $40.31 \pm 1.51$ \\ 
$\text{BADGE}$ & $41.52 \pm 1.04$ & \highlight{$47.84 \pm 0.38$} & $42.40 \pm 1.66$ & $39.33 \pm 1.74$ \\ 
$\text{LHD}$ & \highlight{$43.70 \pm 1.10$} & $46.28 \pm 1.69$ & $41.25 \pm 1.39$ & \highlight{$40.43 \pm 1.92$} \\ 

\midrule
\textbf{Method} & \textbf{\shortstack[l]{Clean}} & \textbf{\shortstack[l]{Noisy-Blank}} & \textbf{\shortstack[l]{Noisy-Diverse}} & \textbf{\shortstack[l]{Noisy-Class}} \\
\midrule
$\text{RAND}+FT$ & $59.64 \pm 0.63$ & $45.45 \pm 0.72$ & $40.31 \pm 1.03$ & $43.12 \pm 0.82$ \\ 
$\text{CONF}+FT$ & $61.40 \pm 1.39$ & $58.66 \pm 2.07$ & $57.09 \pm 2.79$ & $49.78 \pm 2.35$ \\ 
$\text{MARG}+FT$ & $61.91 \pm 0.46$ & $58.33 \pm 0.01$ & $52.94 \pm 2.11$ & $52.46 \pm 24.18$ \\ 
$\text{BALD}+FT$ & $66.64 \pm 0.66$ & $65.31 \pm 0.20$ & $44.34 \pm 1.34$ & $53.00 \pm 0.40$ \\ 
$\text{CORESET}+FT$ & $64.00 \pm 0.90$ & $61.13 \pm 0.31$ & $59.25 \pm 0.06$ & $53.01 \pm 0.15$ \\ 
$\text{BADGE}+FT$ & $64.40 \pm 0.76$ & $64.30 \pm 1.26 $ & $61.36 \pm 0.57$ & $53.78 \pm 0.28$ \\
$\text{LHD}+FT$ & \highlight{$75.35 \pm 0.21$} & \highlight{$75.78 \pm 0.99$} & \highlight{$70.50 \pm 0.56$} & \highlight{$64.34 \pm 0.10$} \\ 

\midrule
\bottomrule
\end{tabular}}
\end{table}

\begin{table}
\centering
\caption{Classification accuracy on SVHN with a Resnet model after 10 rounds of active learning. (\emph{+FT} represents Fine-tuning)}
\label{tab:svhn-results}
\resizebox{0.6\linewidth}{!}{
\begin{tabular}{lcccc}

\toprule
\textbf{Method} &  \textbf{\shortstack[l]{Clean}} & \textbf{\shortstack[l]{Noisy-Blank}} & \textbf{\shortstack[l]{Noisy-Diverse}} & \textbf{\shortstack[l]{Noisy-Class}} \\
\midrule
$\text{RAND}$ & $80.06 \pm 0.20$ & $80.35 \pm 0.46$ & $71.49 \pm 1.30$ & $53.93 \pm 0.98$ \\
$\text{CONF}$ & $77.84 \pm 2.45$ & $75.84 \pm 2.68$ & $76.41 \pm 1.70$ & $30.00 \pm 4.71$ \\
$\text{MARG}$ & $79.11 \pm 0.77$ & $79.84 \pm 1.57$ & $53.82 \pm 3.59$ & $38.45 \pm 3.97$ \\
$\text{BALD}$ & $78.76\pm1.69$ & $84.69\pm2.57$ & $69.10\pm4.13$ & $43.06\pm0.76$ \\
$\text{CORESET}$ & \highlight{$80.71 \pm 0.78$} & $86.97 \pm 0.94$ & \highlight{$86.61 \pm 0.21$} & $65.52 \pm 0.40$ \\
$\text{BADGE}$ & $80.48 \pm 0.81$ & $86.90 \pm 1.46$ & $84.41 \pm 1.01$ & $59.98 \pm 0.55$ \\
$\text{LHD}$ & $78.50 \pm 0.99$ & \highlight{$86.97 \pm 0.12$} & $82.98 \pm 0.67$ & \highlight{$66.01 \pm 0.76$} \\

\midrule
\textbf{Method} & \textbf{\shortstack[l]{Clean}} & \textbf{\shortstack[l]{Noisy-Blank}} & \textbf{\shortstack[l]{Noisy-Diverse}} & \textbf{\shortstack[l]{Noisy-Class}} \\
\midrule
$\text{RAND}+FT$ & $90.60 \pm 0.67$ & $85.9 \pm 0.38$ & $71.49 \pm 1.30$ & $53.93 \pm 0.98$ \\
$\text{CONF}+FT$ & $89.92 \pm 0.81$ & $89.85\pm1.03$ & $83.31 \pm 2.40$ & $57.28 \pm 1.02$ \\
$\text{MARG}+FT$ & $89.73 \pm 1.01$ & $90.76 \pm 0.45$ & $84.74 \pm 2.45$ & $68.12\pm3.79$ \\
$\text{BALD}+FT$ & $91.10 \pm 0.23$ & $90.90 \pm 0.36$ & $67.21 \pm 1.34$ & $63.73\pm4.80$ \\
$\text{CORESET}+FT$ & $87.81 \pm 1.58$ & $89.21 \pm 0.59$ & $87.46 \pm 0.24$ & $70.92 \pm 0.36$ \\
$\text{BADGE}+FT$ & $92.49 \pm 1.12$ & $91.27 \pm 0.87$ & $88.76 \pm 0.63$ & $71.50 \pm 0.22$ \\
$\text{LHD}+FT$ & \highlight{$94.26 \pm 0.12$} & \highlight{$93.54 \pm 0.23$} & \highlight{$90.92 \pm 0.84$} & \highlight{$75.51 \pm 1.09$} \\
\bottomrule
\end{tabular}}
\end{table}

\section{Experiments}
\label{sec:exp}
We experimented with different active learning algorithms, datasets, and noising benchmarks. In all experiments, we started with 2000 labeled points and queried $k = 1000$ examples in each round of active learning, for a total of 10 rounds. We evaluate the performance of two benchmarking datasets - (1) CIFAR10 \cite{Krizhevsky09learningmultiple}, which contains colored images belonging to 10 classes, and (2) SVHN \cite{Netzer2011ReadingDI}, which contains images of street numbers on houses. In all cases, the data pool comprises 80\% noisy data points and 20\% clean data points. Experiments were conducted on the ResNet18 architecture. Our implementation uses the BADGE's codebase \cite{ash2019deep}.


\begin{figure}[!h]
    \centering
    \includegraphics[width=0.6\linewidth]{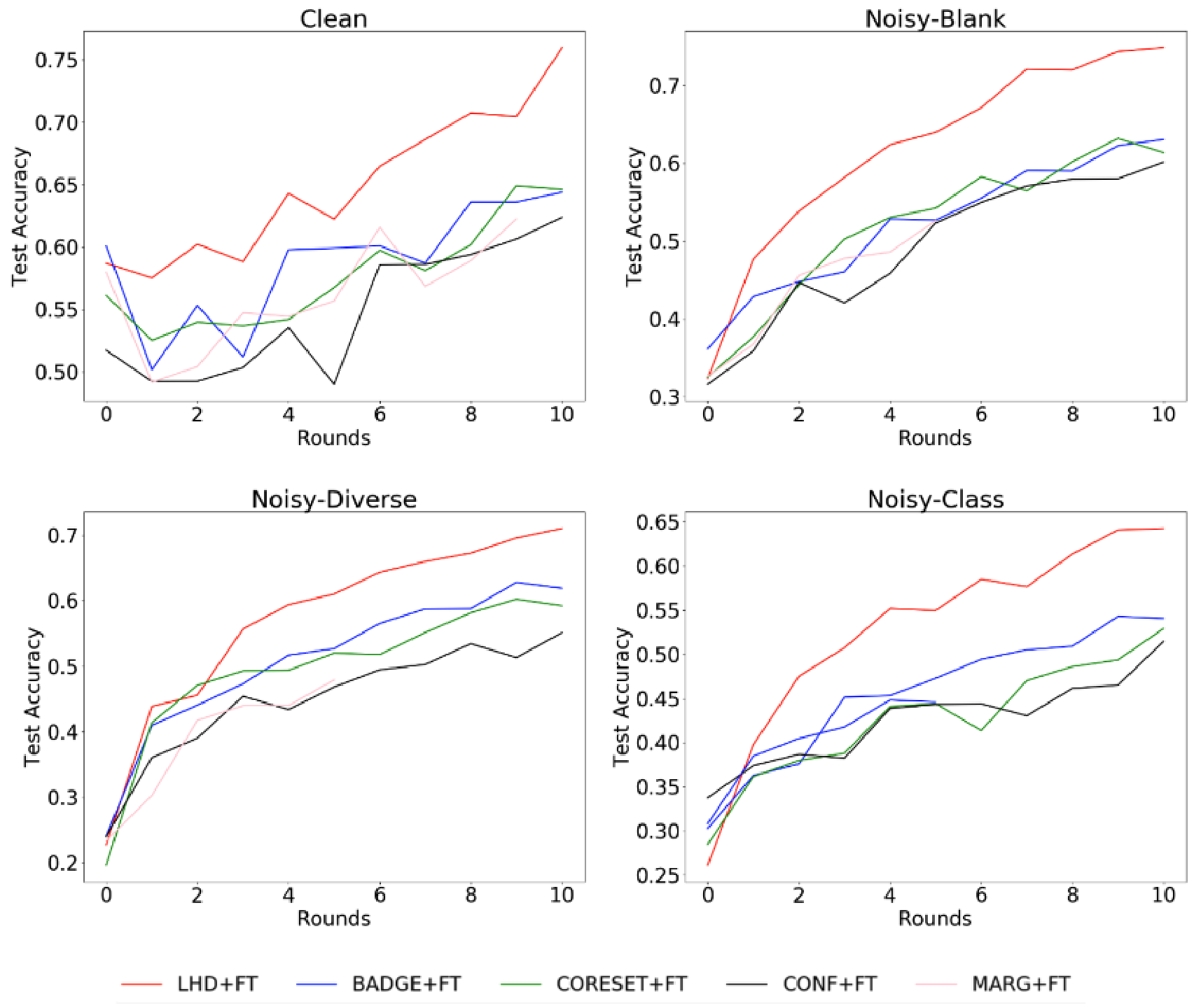}
    \caption{Performance over multiple rounds of active learning acquisition on CIFAR10 dataset. The training is complemented with fine-tuning. \textsc{LHD} outperforms other techniques by a significant margin.}
    \label{fig:cifar10finetune}
\end{figure}


\begin{figure}[!h]
    \centering
    \includegraphics[width=0.6\linewidth]{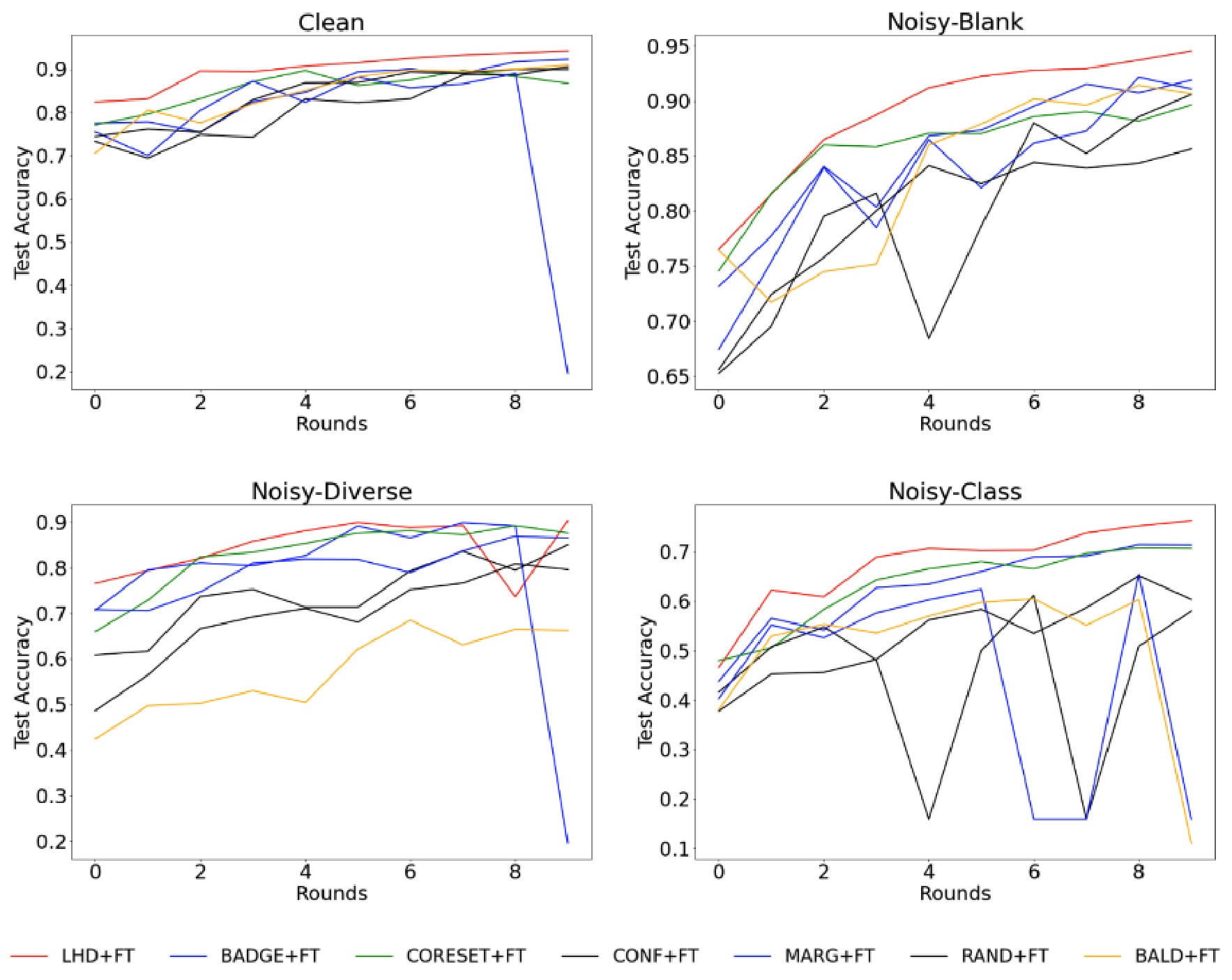}
    \caption{Performance over multiple rounds of active learning acquisition on SVHN dataset. The training is complemented with fine-tuning. \textsc{LHD} outperforms other techniques by a significant margin.}
    \label{fig:svhnfinetune}
\end{figure}

From the analysis of the final results for CIFAR-10 and SVHN shown in Table~\ref{tab:cifar-all} and Table~\ref{tab:svhn-results} respectively, we draw 3 main conclusions.

Firstly, we see that the uncertainty-based techniques (CONF and MARG) perform sub-optimally on the noisy data, even worse than random sampling. Methods that factor in diversity has much better performance. This results are supported by Table \ref{tab:cifar-clean-percentages} and Table \ref{tab:svhn-clean-percentages}, which shows the percentage of clean (non-noisy) examples which are selected over the course of training by different active learning algorithms. The uncertainty-based techniques sample mainly noisy examples, whereas diversity based methods select mainly clean examples. The proposed LHD method has comparable performance to the other diversity based methods.

Secondly, when the fine-tuning technique is added, we observe a significant and consistent boost in the performance of all algorithms across all settings.

Lastly, with fine-tuning, the proposed \textsc{LHD} algorithm outperforms all other techniques. We analyze the performance of the techniques with fine-tuning throughout the 10 active learning rounds, as shown in Figure~\ref{fig:cifar10finetune} and Figure~\ref{fig:svhnfinetune}, which shows that \textsc{LHD} outperforms all the other techniques throughout active learning rounds. Furthermore, Figure~\ref{fig:svhnfinetune} demonstrates catastrophic failure of uncertainty-based active learning techniques in certain rounds of active learning, which happens due to over sampling of noisy examples in those rounds.


Additional experiments for different levels of noise can be found in Table \ref{tab:cifar-different-noise}, and Table \ref{tab:time} shows the time taken by different active learning algorithms to complete 10 rounds of acquisition on a Tesla V100 GPU. We see that \textsc{LHD} takes much less time compared to the diversity-based techniques, viz. BADGE and CORESET.

\section{Theoretical Analysis of Confidence Sampling in Heteroskedastic Setting}
\label{sec:uncertainty-failure}

While previous work \cite{kawaguchi2020ordered} has suggested that selecting high-loss examples can accelerate training, our experiments show that selecting examples with the lowest prediction confidence can fail catastrophically on heteroskedastic distributions. In this section, we provide a theoretical account of why training only on high loss examples (which would have low confidence given a well-calibrated model) can lead to poor performance on heteroskedastic distributions.  

\subsection{Notation}
Let  $\Dcal=((x_i,y_i))_{i=1}^n$ be a training dataset of  $n$ samples where $x_i \in\Xcal \subseteq  \RR^{d_x}$ is the input vector and $y_i \in\Ycal \subseteq  \RR^{d_y}$ is the target vector for the $i$-th sample. A standard objective function is 

\begin{center}
    $L(\theta ;\Dcal) := \frac{1}{n}\sum_{i=1}^n  L_i(\theta;\Dcal)$
\end{center} where $\theta\in\RR^{d_\theta}$ is the parameter vector of the prediction model $f(\hspace{1pt}\cdot\hspace{2pt};\theta): \RR^{d_x}  \rightarrow \RR^{d_y}$, and $L_i(\theta;\Dcal) := \ell(f(x_{i};\theta),y_{i})$ with the function $\ell: \RR^{d_y} \times\Ycal \rightarrow \RR_{\ge 0}$  is the loss of the $i$-th sample. 

Similar to the notation of order statistics, we first introduce the notation of ordered indexes: given a model parameter $\theta$, let $ L_{(1)}(\theta;\Dcal) \ge L_{(2)}(\theta;\Dcal) \ge \cdots \ge L_{(n)}(\theta;\Dcal)$ be the decreasing values of the individual losses $L_1(\theta;\Dcal),\ldots,L_n(\theta;\Dcal)$, where $(j)\in\{1,\dots,n\}$  (for all $j \in \{1,\dots,n\}$). That is, $\{(1),\ldots,(n)\}$ as a perturbation of $\{1,\dots,n\}$ defines the order of sample indexes by loss values. Whenever we  encounter ties on the values, we employ an arbitrary fixed tie-breaking rule in order to ensure the uniqueness of such an order.

Denote $r_i(\theta; \Dcal) =  \sum_{j=1}^{n} \mathbbm{1}\{i=(j)\}\gamma_{j}$ where $(j)$ depends on $(\theta, \Dcal)$. Given an arbitrary set $\Theta\subseteq\RR^{d_\theta}$, we define  $\Rfra_{n}(\Theta)$ as the (standard) Rademacher
complexity of the set $\{ (x,y)\mapsto\ell(f(x;\theta),y): \theta \in \Theta\}$: 

\begin{center}
    $\Rfra_{n}(\Theta) =\EE_{\bar \Dcal,\xi} \left[\sup_{\theta \in \Theta} \frac{1}{n}\sum_{i=1}^n  \xi_{i}\ell(f(\bar x_{i};\theta), \bar y_{i}) \right]$
\end{center}
where $\overline\Dcal=((\bar x_i, \bar y_i))_{i=1}^n$, and $\xi_{1},\dots,\xi_{n}$ are independent uniform random variables taking values in $\{-1,1\}$ (i.e., Rademacher variables). Given a tuple $(\ell, f, \Theta,\Xcal,\Ycal)$, define $M$ as the least upper bound on the difference of individual loss values:  

\begin{center}
    $|\ell(f(x;\theta),y)-\ell(f(x';\theta),y')| \le M $
\end{center}
for all $\theta \in\Theta $ and all  $(x,y),(x',y') \in \Xcal \times \Ycal$. For example, $M=1$ if $\ell$ is the  0-1 loss function.  We can then write: 

\begin{center}
    $\hat \Rfra_{n}(\Theta) =\EE_{\xi} \left[\sup_{\theta \in \Theta} \frac{1}{n}\sum_{i=1}^n  \xi_{i}\ell(f(x_{i};\theta), y_{i}) \right]$. 
\end{center}

\subsection{Preliminaries}
The previous paper \cite{kawaguchi2020ordered} proves that the stochastic optimization method that uses a gradient estimator that is purposely biased toward those samples with the current top-$q$  losses (i.e., ordered SGD)  implicitly minimizes a new  objective function of $L_{q}(\theta;\Dcal) = \frac{1}{q} \sum_{j=1}^n \gamma_{j} L_{(j)}(\theta;\Dcal)$, for any $\Dcal$ (including $g(\Dcal)$), in the sense that such a  gradient estimator is an unbiased estimator of a (sub-) gradient of $L_q(\theta;\Dcal)$, instead of $L(\theta;\Dcal)$. Accordingly, the top-$q$-biased stochastic optimization method converges in terms of $L_q$ instead of $L$. 

Building up on this result, we consider generalization properties of the top-$q$-biased stochastic optimization  with the presence of additional label noises in training data. We want to minimize the expected loss, $\EE_{(x,y)\sim \Pcal}[\ell(f(x;\theta),y)]$, by minimizing the training loss $L_{q}(\theta;g(\Dcal))$, where $g(\mathcal{D})=((\gx(x_i),\gy(y_i)))_{i=1}^n$ is potentially corrupted by arbitrary noise and corruption effects within arbitrary fixed functions $\gx$ and $\gy$ for $i=1\dots,n$, where
$(x_i,y_i)\sim \Pcal$. Thus, we want to analyze the generalization gap: 

\begin{center}
    $\EE_{(x,y)\sim \Pcal}[\ell(f(x;\theta),y)]- L_{q}(\theta;g(\Dcal))$
\end{center}  

\cite{kawaguchi2020ordered} showed the benefit of the top-$q$-biased stochastic optimization method in terms of generalization when $\gx$ and $\gy$ are identity functions and thus when the distributions are the same for both expected loss and training loss.  In contrast, in our setting,  the  distributions are different for expected loss and training loss with potential noise corruptions through  $\gx$ and $\gy$.

\subsection{Generalization Bound for Biased Query Samples}


\begin{thm} \label{thm:1}
Let $\Theta$ be a fixed subset of $\RR^{d_\theta}$. Then, for any $\delta>0$, with probability at least $1-\delta$ over an iid draw of $n$  examples $\mathcal{D}=((x_i,y_i))_{i=1}^n$, the following holds for all $\theta \in \Theta$:


\begin{center}
\label{eq:ge_1a}
  $\EE_{(x,y)}[\ell(f(x;\theta),y)]
    \le L_{q}(\theta;g(\Dcal)) + 2\hat \Rfra_{n}(\Theta) +
    M\left(2+\frac{s}{q}\right) \sqrt{\frac{\ln (2/\delta)}{2n}} - \Qcal_{n,q}(\Theta,g)$  
\end{center} where we define the top-$q$-biased factor as 

\begin{center}
    $\Qcal_{n,q}(\Theta,g) := \EE_{\bar \Dcal} \Big[\inf_{\theta \in \Theta} \frac{1}{n} \sum_{i=1}^n \big[\frac{r_i(\theta;g(\bar \Dcal))n}{q} \cdot  \ell(f(\gx(\bar x_{i});\theta), \gy(\bar y_{i}))-\ell(f(\bar x_{i};\theta), \bar y_{i}) \big] \Big].$
\end{center}
\end{thm}

The expected error $\EE_{(x,y)}[\ell(f(x;\theta),y)]$ in the left hand side of Theorem~\ref{thm:1} is a standard objective for generalization, whereas the right-hand side contains the data corruption function $g$. Here, we typically have $\Rfra_{n}(\Theta)=O(1/\sqrt{n})$ in terms of $n$.  For example,  consider the standard feedforward deep neural networks of  the form $f(x)= (\omega_T \circ \sigma_{T-1} \circ \omega_{T-1} \circ \sigma_{T-2}\cdots \sigma_1 \circ \omega_1)(x)$ where $T$ is the number of  layers, $\omega_l(a)=W_{l}a$ with $\|W_{l}\|_F \le M_l$, and $\sigma_{l}$ is an element-wise nonlinear activation function that is  1-Lipschitz and positive homogeneous (e.g., ReLU).
Then, if  $\|x\|\le B$ for all $x \in \Xcal$, using Theorem 1 of \cite{golowich2018size}, we have that: 

\begin{center}
    $\hat \Rfra_{n}(\Theta)\le \frac{B (\sqrt{2 \log(2)T }+1)(\prod_{l=1}^T M_l)}{\sqrt{n}}$
\end{center} 

In Theorem \ref{thm:1}, we can see that a label noise corruption $g$ can lead to the failure of the top-$q$-biased stochastic optimization
via increasing the training loss $L_{q}(\theta;g(\Dcal))$ and decreasing the top-$q$-biased factor $\Qcal_{n,q}(\Theta,g)$. Here, if there is no corruption $g$ (i.e., if $\gx$ and $\gy$ are identity functions), then we have that $\Qcal_{n,q}(\Theta,g)\ge 0$ because $\Qcal_{n,q}(\Theta,g)=\EE_{\bar \Dcal} [\inf_{\theta \in \Theta} L_{q}(\theta;\bar \Dcal)-L(\theta ;\bar \Dcal)] \ge 0$ due to $L_{q}(\theta;\bar \Dcal)-L(\theta ;\bar \Dcal)\ge 0$ for any $\theta$ and $\bar \Dcal$ when $\gx$ and $\gy$ are identity functions. Thus, the top-$q$-biased factor $\Qcal_{n,q}(\Theta,g)$ can explain the improvement of the generalization of the top-$q$-biased stochastic optimization over the standard unbiased stochastic optimization. However, with the presence of  the corruption $g$,  $\frac{r_i(\theta;g(\bar \Dcal))n}{q}\ell(f(\gx(\bar x_{i});\theta), \gy(\bar y_{i}))$ can be smaller than $\ell(f(\bar x_{i};\theta), \bar y_{i})$ \textit{by fitting the corrupted noise}, resulting  $\Qcal_{n,q}(\Theta,g)<0$. This leads to a significant failure in the following sense: 

The generalization gap ($\EE_{(x,y)}[\ell(f(x;\theta),y)]
-L_{q}(\theta;g(\Dcal)) $) goes to zero as $n$ approach infinity if $\Qcal_{n,q}(\Theta,g)\ge 0$ with no data corruption, but the generalization gap no longer goes to zero as as $n$ approach infinity if $\Qcal_{n,q}(\Theta,g)< 0$  with  data corruption.

To see this, let us  look at the asymptotic case when  $n\rightarrow\infty$. Let $\Theta$ be constrained such that $\Rfra_n(\Theta)\rightarrow 0 $ as $n\rightarrow \infty$, which has been shown to be satisfied for various models and sets $\Theta$, including the standard deep neural networks above \cite{bartlett2002rademacher,mohri2012foundations,bartlett2017spectrally,kawaguchi2017generalization,golowich2018size}. The third term in the right-hand side of the Equation in Theorem~\ref{thm:1} disappears as $n\rightarrow\infty$. Thus, if there is no corruption (i.e., if $\gx$ and $\gy$ are identity functions), it holds with high probability that  

\begin{center}
    $\EE_{(x,y)}[\ell(f(x;\theta),y)] \le L_{q}(\theta;g(\Dcal))-\Qcal_{n,q}(\Theta,g)\le L_{q}(\theta;g(\Dcal))$ 
\end{center} where $L_{q}(\theta;g(\Dcal))$ is minimized by the top-$q$-biased stochastic optimization. From this viewpoint, the top-$q$-biased stochastic optimization minimizes the expected error for generalization when $n\rightarrow\infty$, if there is no corruption. However, if there is  corruption, 

\begin{center}
    $\EE_{(x,y)}[\ell(f(x;\theta),y)] \le L_{q}(\theta;g(\Dcal))-\Qcal_{n,q}(\Theta,g)\nleq\ L_{q}(\theta;g(\Dcal))$ 
\end{center}

Hence $\EE_{(x,y)}[\ell(f(x;\theta),y)]-L_{q}(\theta;g(\Dcal)) \nrightarrow 0$ even in the asymptotic case. The full proof of Theorem~\ref{thm:1} can be found in Section \ref{sec:proof} of the Appendix.
\section{Conclusion}
Neural Active Learning is an active area of research, with many new techniques competing to achieve better results. Our work seeks to challenge the commonly held assumption that the training data is independent and identically distributed (I.I.D) for the active learning setup. We show that the uncertainty-based techniques that are competitive on homogeneous datasets with little label noise can fail catastrophically when presented with diverse heteroskedastic distributions. We also explore the different techniques (diversity-based sampling, fine-tuning, and LHD) that can be used to mitigate these failures. We believe that research has to be done exploring the various possible data distributions, since there is no guarantee that the I.I.D assumption holds for real-world data, and develop algorithms that are robust to the different data distributions.


\bibliographystyle{unsrt}  
\bibliography{references}

\begin{thebibliography}{10}

\bibitem{Wang95quality}
R.Y. Wang, V.C. Storey, and C.P. Firth.
\newblock A framework for analysis of data quality research.
\newblock {\em IEEE Transactions on Knowledge and Data Engineering},
  7(4):623--640, 1995.

\bibitem{self-instruct}
Yizhong Wang, Yeganeh Kordi, Swaroop Mishra, Alisa Liu, Noah~A. Smith, Daniel
  Khashabi, and Hannaneh Hajishirzi.
\newblock Self-instruct: Aligning language model with self generated
  instructions, 2022.

\bibitem{alpaca}
Rohan Taori, Ishaan Gulrajani, Tianyi Zhang, Yann Dubois, Xuechen Li, Carlos
  Guestrin, Percy Liang, and Tatsunori~B. Hashimoto.
\newblock Stanford alpaca: An instruction-following llama model, 2023.

\bibitem{llama-adapter}
Renrui Zhang, Jiaming Han, Aojun Zhou, Xiangfei Hu, Shilin Yan, Pan Lu,
  Hongsheng Li, Peng Gao, and Yu~Qiao.
\newblock Llama-adapter: Efficient fine-tuning of language models with
  zero-init attention, 2023.

\bibitem{kawaguchi2020ordered}
Kenji Kawaguchi and Haihao Lu.
\newblock Ordered sgd: A new stochastic optimization framework for empirical
  risk minimization.
\newblock In {\em International Conference on Artificial Intelligence and
  Statistics}, pages 669--679, 2020.

\bibitem{D11}
Sanjoy Dasgupta.
\newblock Two faces of active learning.
\newblock {\em Theoretical computer science}, 2011.

\bibitem{chaudhuri2015convergence}
Kamalika Chaudhuri, Sham Kakade, Praneeth Netrapalli, and Sujay Sanghavi.
\newblock Convergence rates of active learning for maximum likelihood
  estimation.
\newblock In {\em Advances in Neural Information Processing Systems}, 2015.

\bibitem{chaudhuri2017active}
Kamalika Chaudhuri, Prateek Jain, and Nagarajan Natarajan.
\newblock Active heteroscedastic regression.
\newblock In {\em International Conference on Machine Learning}, pages
  694--702. PMLR, 2017.

\bibitem{tur2005combining}
Gokhan Tur, Dilek Hakkani-T{\"u}r, and Robert~E Schapire.
\newblock Combining active and semi-supervised learning for spoken language
  understanding.
\newblock {\em Speech Communication}, 2005.

\bibitem{gal2017deep}
Yarin Gal, Riashat Islam, and Zoubin Ghahramani.
\newblock Deep bayesian active learning with image data.
\newblock In {\em International Conference on Machine Learning}, 2017.

\bibitem{sener2018active}
Ozan Sener and Silvio Savarese.
\newblock Active learning for convolutional neural networks: A core-set
  approach.
\newblock In {\em International Conference on Learning Representations}, 2018.

\bibitem{geifman2017deep}
Yonatan Geifman and Ran El-Yaniv.
\newblock Deep active learning over the long tail.
\newblock {\em arXiv:1711.00941}, 2017.

\bibitem{gissin2019discriminative}
Daniel Gissin and Shai Shalev-Shwartz.
\newblock Discriminative active learning.
\newblock {\em arXiv:1907.06347}, 2019.

\bibitem{wang2015querying}
Zheng Wang and Jieping Ye.
\newblock Querying discriminative and representative samples for batch mode
  active learning.
\newblock {\em Transactions on Knowledge Discovery from Data}, 2015.

\bibitem{pmlr-v28-chen13b}
Yuxin Chen and Andreas Krause.
\newblock Near-optimal batch mode active learning and adaptive submodular
  optimization.
\newblock In {\em International Conference on Machine Learning}, 2013.

\bibitem{wei2015submodularity}
Kai Wei, Rishabh Iyer, and Jeff Bilmes.
\newblock Submodularity in data subset selection and active learning.
\newblock In {\em International Conference on Machine Learning}, 2015.

\bibitem{batchbald}
Andreas Kirsch, Joost van Amersfoort, and Yarin Gal.
\newblock Batchbald: Efficient and diverse batch acquisition for deep bayesian
  active learning.
\newblock In {\em Advances in Neural Information Processing Systems}, 2019.

\bibitem{ash2021gone}
Jordan~T Ash, Surbhi Goel, Akshay Krishnamurthy, and Sham Kakade.
\newblock Gone fishing: Neural active learning with fisher embeddings.
\newblock {\em arXiv preprint arXiv:2106.09675}, 2021.

\bibitem{ash2019deep}
Jordan~T Ash, Chicheng Zhang, Akshay Krishnamurthy, John Langford, and Alekh
  Agarwal.
\newblock Deep batch active learning by diverse, uncertain gradient lower
  bounds.
\newblock {\em International Conference on Learning Representations}, 2020.

\bibitem{steinhardt2017certified}
Jacob Steinhardt, Pang~Wei Koh, and Percy Liang.
\newblock Certified defenses for data poisoning attacks.
\newblock In {\em Proceedings of the 31st International Conference on Neural
  Information Processing Systems}, pages 3520--3532, 2017.

\bibitem{sagawa2019distributionally}
Shiori Sagawa, Pang~Wei Koh, Tatsunori~B Hashimoto, and Percy Liang.
\newblock Distributionally robust neural networks.
\newblock In {\em International Conference on Learning Representations}, 2019.

\bibitem{lin2021active}
Jing Lin, Ryan Luley, and Kaiqi Xiong.
\newblock Active learning under malicious mislabeling and poisoning attacks.
\newblock {\em arXiv preprint arXiv:2101.00157}, 2021.

\bibitem{vicarte2021double}
Jose Rodrigo~Sanchez Vicarte, Gang Wang, and Christopher~W. Fletcher.
\newblock Double-cross attacks: Subverting active learning systems.
\newblock In {\em 30th {USENIX} Security Symposium ({USENIX} Security 21)},
  pages 1593--1610. {USENIX} Association, August 2021.

\bibitem{similar2021al}
Suraj Kothawade, Nathan Beck, Krishnateja Killamsetty, and Rishabh Iyer.
\newblock Similar: Submodular information measures based active learning in
  realistic scenarios.
\newblock In M.~Ranzato, A.~Beygelzimer, Y.~Dauphin, P.S. Liang, and J.~Wortman
  Vaughan, editors, {\em Advances in Neural Information Processing Systems},
  volume~34, pages 18685--18697. Curran Associates, Inc., 2021.

\bibitem{contrastive2021al}
Pan Du, Suyun Zhao, Hui Chen, Shuwen Chai, Hong Chen, and Cuiping Li.
\newblock Contrastive coding for active learning under class distribution
  mismatch.
\newblock In {\em 2021 IEEE/CVF International Conference on Computer Vision
  (ICCV)}, pages 8907--8916, 2021.

\bibitem{shu2019meta}
Jun Shu, Qi~Xie, Lixuan Yi, Qian Zhao, Sanping Zhou, Zongben Xu, and Deyu Meng.
\newblock Meta-weight-net: Learning an explicit mapping for sample weighting.
\newblock {\em arXiv preprint arXiv:1902.07379}, 2019.

\bibitem{cao2020heteroskedastic}
Kaidi Cao, Yining Chen, Junwei Lu, Nikos Arechiga, Adrien Gaidon, and Tengyu
  Ma.
\newblock Heteroskedastic and imbalanced deep learning with adaptive
  regularization.
\newblock {\em arXiv preprint arXiv:2006.15766}, 2020.

\bibitem{ANTOS20102712}
András Antos, Varun Grover, and Csaba Szepesvári.
\newblock Active learning in heteroscedastic noise.
\newblock {\em Theoretical Computer Science}, 411(29):2712--2728, 2010.
\newblock Algorithmic Learning Theory (ALT 2008).

\bibitem{Zhu2003CombiningAL}
Xiaojin Zhu, John~D. Lafferty, and Zoubin Ghahramani.
\newblock Combining active learning and semi-supervised learning using gaussian
  fields and harmonic functions.
\newblock In {\em ICML 2003}, 2003.

\bibitem{consistencysslal}
Mingfei Gao, Zizhao Zhang, Guo Yu, Sercan~O. Arik, Larry~S. Davis, and Tomas
  Pfister.
\newblock Consistency-based semi-supervised active learning: Towards minimizing
  labeling cost, 2019.

\bibitem{bilevelopt}
Zalán Borsos, Marco Tagliasacchi, and Andreas Krause.
\newblock Semi-supervised batch active learning via bilevel optimization, 2020.

\bibitem{ssl2019asr}
Thomas Drugman, Janne Pylkkonen, and Reinhard Kneser.
\newblock Active and semi-supervised learning in asr: Benefits on the acoustic
  and language models.
\newblock 2019.

\bibitem{Rhee2017ActiveAS}
Phill~K. Rhee, Enkhbayar Erdenee, Shin~Dong Kyun, Minhaz~Uddin Ahmed, and
  SongGuo Jin.
\newblock Active and semi-supervised learning for object detection with
  imperfect data.
\newblock {\em Cognitive Systems Research}, 45:109--123, 2017.

\bibitem{Sinha2019vaal}
Samarth Sinha, Sayna Ebrahimi, and Trevor Darrell.
\newblock Variational adversarial active learning, 2019.

\bibitem{wang2014new}
Dan Wang and Yi~Shang.
\newblock A new active labeling method for deep learning.
\newblock In {\em International Joint Conference on Neural Networks}, 2014.

\bibitem{roth2006margin}
Dan Roth and Kevin Small.
\newblock Margin-based active learning for structured output spaces.
\newblock In {\em European Conference on Machine Learning}, 2006.

\bibitem{bald2011}
Neil Houlsby, Ferenc Huszár, Zoubin Ghahramani, and Máté Lengyel.
\newblock Bayesian active learning for classification and preference learning,
  2011.

\bibitem{Arthur07kmeans}
David Arthur and Sergei Vassilvitskii.
\newblock K-means++: The advantages of careful seeding.
\newblock In {\em Proceedings of the Eighteenth Annual ACM-SIAM Symposium on
  Discrete Algorithms}, SODA '07, page 1027–1035, USA, 2007. Society for
  Industrial and Applied Mathematics.

\bibitem{sohn2020fixmatch}
Kihyuk Sohn, David Berthelot, Nicholas Carlini, Zizhao Zhang, Han Zhang,
  Colin~A Raffel, Ekin~Dogus Cubuk, Alexey Kurakin, and Chun-Liang Li.
\newblock Fixmatch: Simplifying semi-supervised learning with consistency and
  confidence.
\newblock {\em Advances in Neural Information Processing Systems}, 33:596--608,
  2020.

\bibitem{Krizhevsky09learningmultiple}
Alex Krizhevsky.
\newblock Learning multiple layers of features from tiny images.
\newblock Technical report, 2009.

\bibitem{Netzer2011ReadingDI}
Yuval Netzer, Tiejie Wang, Adam Coates, A.~Bissacco, Bo~Wu, and A.~Ng.
\newblock Reading digits in natural images with unsupervised feature learning.
\newblock 2011.

\bibitem{golowich2018size}
Noah Golowich, Alexander Rakhlin, and Ohad Shamir.
\newblock Size-independent sample complexity of neural networks.
\newblock In {\em Conference On Learning Theory}, pages 297--299. PMLR, 2018.

\bibitem{bartlett2002rademacher}
Peter~L Bartlett and Shahar Mendelson.
\newblock Rademacher and gaussian complexities: Risk bounds and structural
  results.
\newblock {\em Journal of Machine Learning Research}, 3(Nov):463--482, 2002.

\bibitem{mohri2012foundations}
Mehryar Mohri, Afshin Rostamizadeh, and Ameet Talwalkar.
\newblock {\em Foundations of machine learning}.
\newblock MIT press, 2012.

\bibitem{bartlett2017spectrally}
Peter~L Bartlett, Dylan~J Foster, and Matus~J Telgarsky.
\newblock Spectrally-normalized margin bounds for neural networks.
\newblock In {\em Advances in Neural Information Processing Systems}, pages
  6240--6249, 2017.

\bibitem{kawaguchi2017generalization}
Kenji Kawaguchi, Leslie~Pack Kaelbling, and Yoshua Bengio.
\newblock Generalization in deep learning.
\newblock {\em In Mathematics of Deep Learning, Cambridge University Press, to
  appear. Prepint available as: MIT-CSAIL-TR-2018-014, Massachusetts Institute
  of Technology}, 2018.

\end{thebibliography}

\clearpage

\appendix

\section{Additional Experiments}
\textbf{Percentage of clean examples selected from the unlabeled pool: }
We investigated the fraction of clean (non-noisy) examples which are selected over the course of training (Table~\ref{tab:cifar-clean-percentages} and \ref{tab:svhn-clean-percentages}). The algorithms that do not factor in diversity (viz., CONF and MARG) while selecting examples end up selecting a very small fraction for clean examples and perform no better than random sampling. However, techniques that favor diversity select a large fraction of clean examples. For example, each of CORESET, BADGE, and LHD almost perfectly filter out the \emph{Noisy-Blank} examples.

\begin{table}
\centering
\caption{Percentage of clean samples selected for CIFAR10}
\label{tab:cifar-clean-percentages}
\resizebox{0.6\linewidth}{!}{
\begin{tabular}{lcccc}

\toprule
\textbf{Resnet without fine-tuning} & \textbf{\shortstack[l]{Noisy-Blank}} & \textbf{\shortstack[l]{Noisy-Diverse}} & \textbf{\shortstack[l]{Noisy-Class}} \\
\midrule
$\text{RAND}$ & $20.25$ & $20.53$ & $17.96$ \\
$\text{CONF}$ & $10.41$ & $46.99$ & $11.42$ \\ 
$\text{MARG}$ & $21.98$ & $4.61$ & $10.83$ \\ 
$\text{BALD}$ & $80.62$ & $18.62$ & $12.47$ \\ 
$\text{CORESET}$ & $100.00$ & $100.00$ & $94.25$ \\ 
$\text{BADGE}$ & $99.90$ & $90.0$ & $80.70$ \\ 
$\text{LHD}$ & $99.90$ & $53.97$ & $82.55$ \\ 

\midrule
\textbf{Resnet with fine-tuning} & \textbf{\shortstack[l]{Noisy-Blank}} & \textbf{\shortstack[l]{Noisy-Diverse}} & \textbf{\shortstack[l]{Noisy-Class}} \\
\midrule
$\text{RAND}$ & $19.70$ & $19.20$ & $16.50$ \\ 
$\text{CONF}$ & $100.00$ & $83.16$ & $87.84$ \\ 
$\text{MARG}$ & $100.00$ & $58.34$  & $89.08$\\ 
$\text{BALD}$ & $99.90$ & $12.68$ & $79.85$ \\ 
$\text{CORESET}$ & $100.00$ & $100.00$ & $81.65$ \\ 
$\text{BADGE}$& $99.90$ & $92.705$ & $87.41$ \\
$\text{LHD}$ & $99.90$ & $77.88$ & $82.80$ \\ 

\midrule
\bottomrule
\end{tabular}}
\end{table}

\begin{table}
\centering
\caption{Percentage of clean samples selected for SVHN}
\label{tab:svhn-clean-percentages}
\resizebox{0.6\linewidth}{!}{
\begin{tabular}{lcccc}

\toprule
\textbf{ResNet without fine-tuning} & \textbf{\shortstack[l]{Noisy-Blank}} & \textbf{\shortstack[l]{Noisy-Diverse}} & \textbf{\shortstack[l]{Noisy-Class}} \\
\midrule
$\text{RAND}$ & $19.84$ & $20.16$ & $19.81$ \\
$\text{CONF}$ & $10.00$ & $37.59$ & $10.46$ \\ 
$\text{MARG}$ & $17.26$ & $7.52$ & $8.01$ \\ 
$\text{BALD}$ & $69.09$ & $25.58$ & $13.12$ \\ 
$\text{CORESET}$ & $100.00$ & $100.00$ & $100.00$ \\ 
$\text{BADGE}$ & $99.90$ & $90.00$ & $90.00$ \\ 
$\text{LHD}$ & $99.90$ & $77.80$ & $76.71$ \\ 

\midrule
\textbf{ResNet with fine-tuning} & \textbf{\shortstack[l]{Noisy-Blank}} & \textbf{\shortstack[l]{Noisy-Diverse}} & \textbf{\shortstack[l]{Noisy-Class}} \\
\midrule
$\text{RAND}$ & $20.00$ & $20.01$ & $16.36$ \\ 
$\text{CONF}$ & $100.00$ & $59.20$ & $60.44$ \\ 
$\text{MARG}$ & $100.00$ & $24.77$  & $66.01$\\ 
$\text{BALD}$ & $98.47$ & $11.27$  & $40.60$\\ 
$\text{CORESET}$ & $100.00$ & $100.00$ & $71.35$ \\ 
$\text{BADGE}$& $99.90$ & $91.44$ & $73.30$ \\
$\text{LHD}$ & $99.90$ & $58.06$ & $69.83$ \\ 




\midrule
\bottomrule
\end{tabular}}
\end{table}

A point to be noted is that an algorithm that samples more clean examples might not necessarily outperform other algorithms that sample less clean examples. For instance, while CORESET samples more clean examples than LHD in the \emph{Noisy-Class} setup, LHD outperforms CORESET (Table~\ref{tab:cifar-all}). This is because sampling the clean examples is just one part of the problem. The clean examples also have to optimize the model's performance. Even though LHD samples lesser clean examples, it selects higher quality clean examples, leading to better overall performance.

Another point is that the CORESET and BADGE algorithms are able to achieve high levels of clean percentages because the number of unique noisy examples in the heteroskedastic benchmarks are limited. Future work can be done exploring more challenging heteroskedastic datasets.

\textbf{Different levels of noise: }
We ran some ablation studies to investigate the performance of the active learning algorithms under milder noise conditions. Without any loss of generality, Table~\ref{tab:cifar-different-noise} shows the results for different methods on out most challenging dataset - \emph{Noisy-Class} CIFAR10.

\begin{table}
\centering
\caption{Classification accuracy on \emph{Noisy-Class} CIFAR10 setup for different percentages of noisy examples (all methods use fine-tuning on unlabeled data).}
\label{tab:cifar-different-noise}
\resizebox{0.5\linewidth}{!}{
\begin{tabular}{lcccc}

\toprule
\textbf{Method} & \textbf{\shortstack[l]{20\% Noise}} & \textbf{\shortstack[l]{40\% Noise}} & \textbf{\shortstack[l]{80\% Noise}} \\
\midrule
$\text{RAND}$ & $57.24\pm0.40$ & $53.39\pm0.84$ & $43.12\pm0.82$ \\
$\text{CONF}$ & $55.00\pm0.63$ & $53.66\pm2.64$ & $49.87\pm2.35$ \\ 
$\text{MARG}$ & $58.05\pm0.61$ & $57.83\pm0.39$ & $52.46\pm24.18$ \\ 
$\text{CORESET}$ & $58.21\pm0.47$ & $58.66\pm1.03$ & $53.01\pm0.15$ \\ 
$\text{BADGE}$ & $58.36\pm1.90$ & $56.97\pm1.10$ & $53.78\pm0.28$ \\ 
$\text{LHD}$ & \highlight{$68.06\pm0.34$} & \highlight{$67.15\pm0.55$} & \highlight{$64.34\pm0.10$} \\ 

\bottomrule
\end{tabular}}
\end{table}

\textbf{Conflicting gradients for noisy examples: }
In Section \ref{lhd}, we posit that (1) for noisy examples, the conflicting gradients result in the model converging quickly to the suboptimal solution and undergoing little change throughout the training, and (2) for the clean examples, the model learns the correct solution and converges to an optimal solution by undergoing changes throughout the training. 

To substantiate this claim, we investigate the loss and the average gradient magnitude for noisy and clean examples (Figure \ref{fig:noisy-vs-clean-loss} and \ref{fig:conflicting-gradients}, respectively). Since one noisy example can be mapped to more than one label, the model will observe conflicting gradients for these examples during training. So, the model quickly converges to a suboptimal solution for these examples. For clean examples, on the other hand, there is a 1-to-1 image-to-label mapping that is learned by the model as it trains.

\begin{figure}[ht!]
    \centering
    \includegraphics[width=0.75\linewidth]{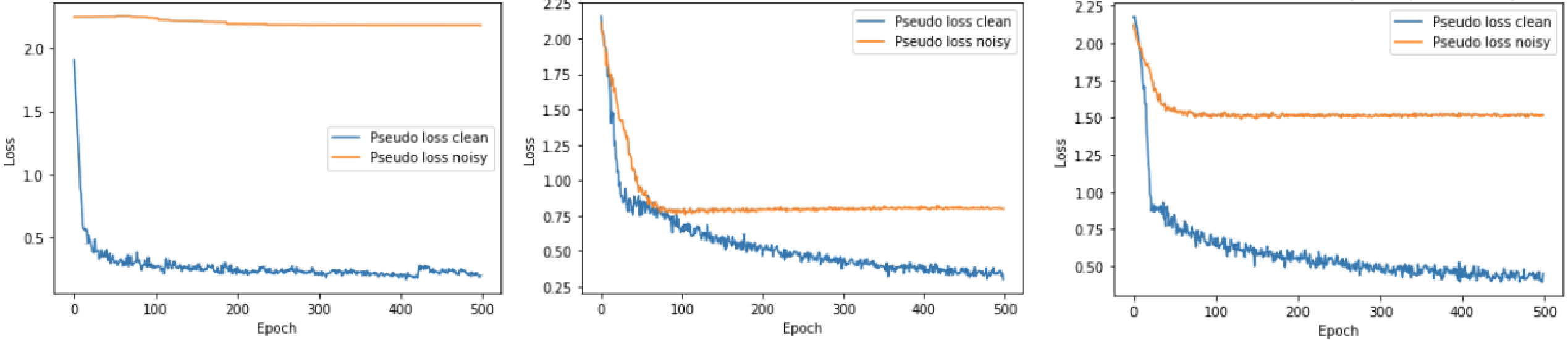}
    \caption{Loss curves for noisy examples (orange) and clean examples (blue) in the first round of training on CIFAR-10 using ResNet for the three setups (noisy-blank, noisy-diverse, and noisy-class from left to right). As can be seen, the loss quickly converges to a suboptimal solution for noisy examples, while it continues to drop gradually for clean examples.}
    \label{fig:noisy-vs-clean-loss}
\end{figure}

\begin{figure}[ht!]
    \centering
    \includegraphics[width=0.75\linewidth]{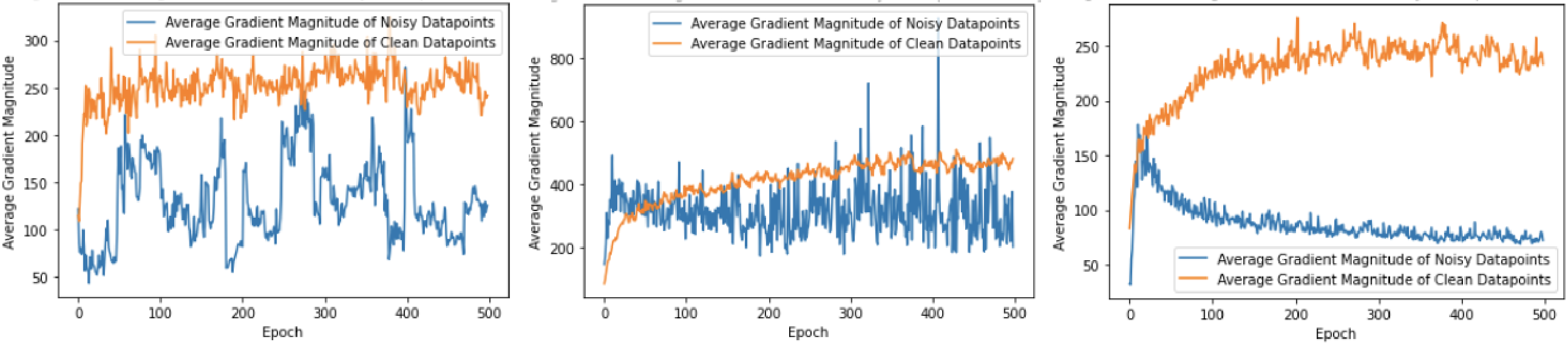}
    \caption{Average gradient magnitudes for noisy examples (orange) and clean examples (blue) in the first round of training on CIFAR-10 using ResNet for the three setups (noisy-blank, noisy-diverse, and noisy-class from left to right). As can be seen, the gradient magnitude varies significantly for the noisy examples compared to the clean examples.}
    \label{fig:conflicting-gradients}
\end{figure}

\textbf{Effect of fine-tuning on \textsc{LHD}: }
As can be seen from Figure \ref{fig:lh-diff}, on supplementing LHD with fine-tuning, the average $||\mathbf{lh}||_2$ for the \emph{difficult-clean} examples becomes significantly higher than the $||\mathbf{lh}||_2$ for \emph{simple-clean} examples. This makes LHD adept at differentiating \emph{difficult-clean} examples from \emph{simple-clean} examples, which improves the performance of the active learning algorithm.

\begin{figure}[!ht]
    \centering
    \begin{subfigure}{.38\textwidth}
    \includegraphics[width=0.99\linewidth]{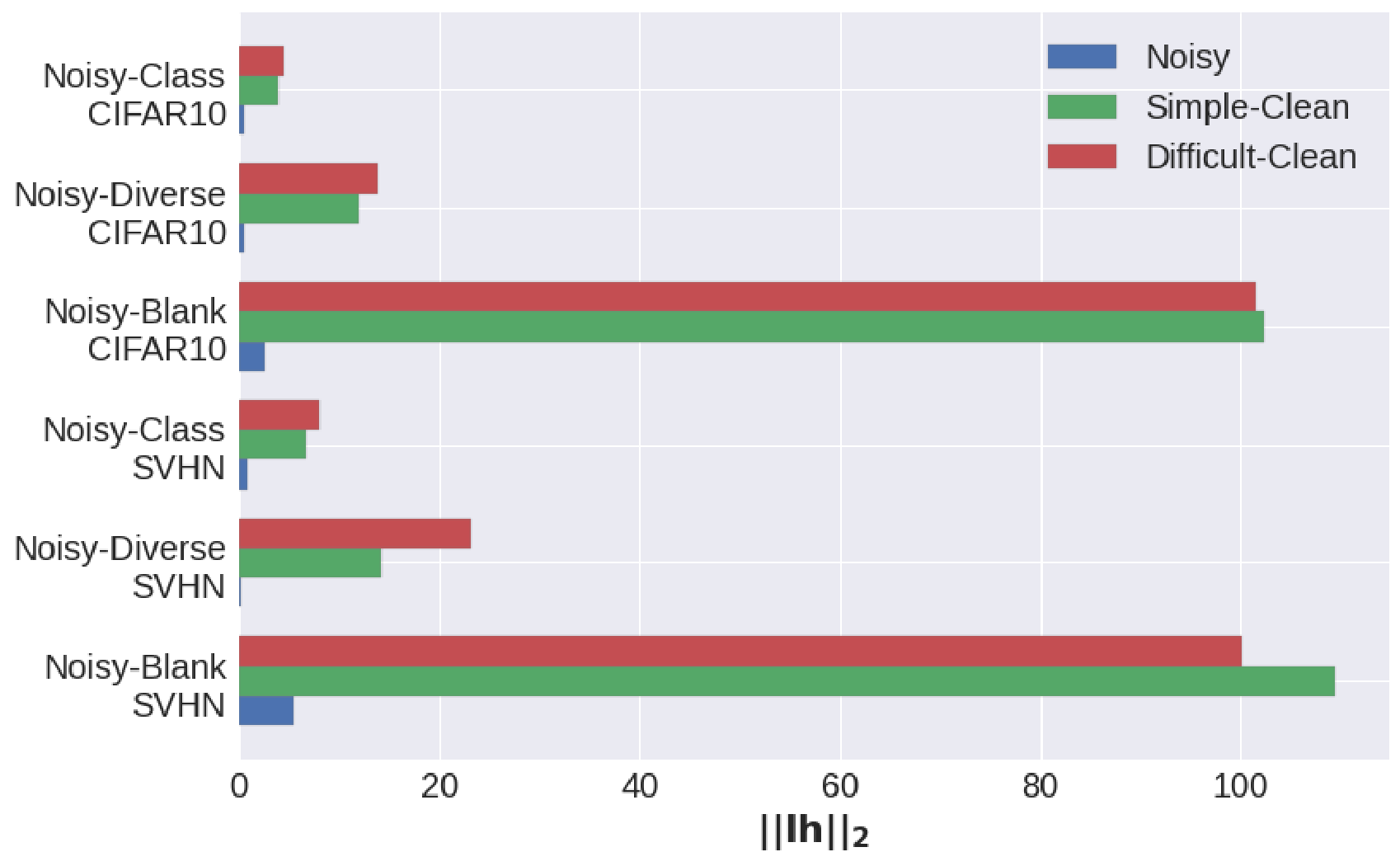}
    \caption{Without fine-tuning}
    \label{fig:wo-ft}
    \end{subfigure}
    \centering
    \begin{subfigure}{.38\textwidth}
    \includegraphics[width=0.99\linewidth]{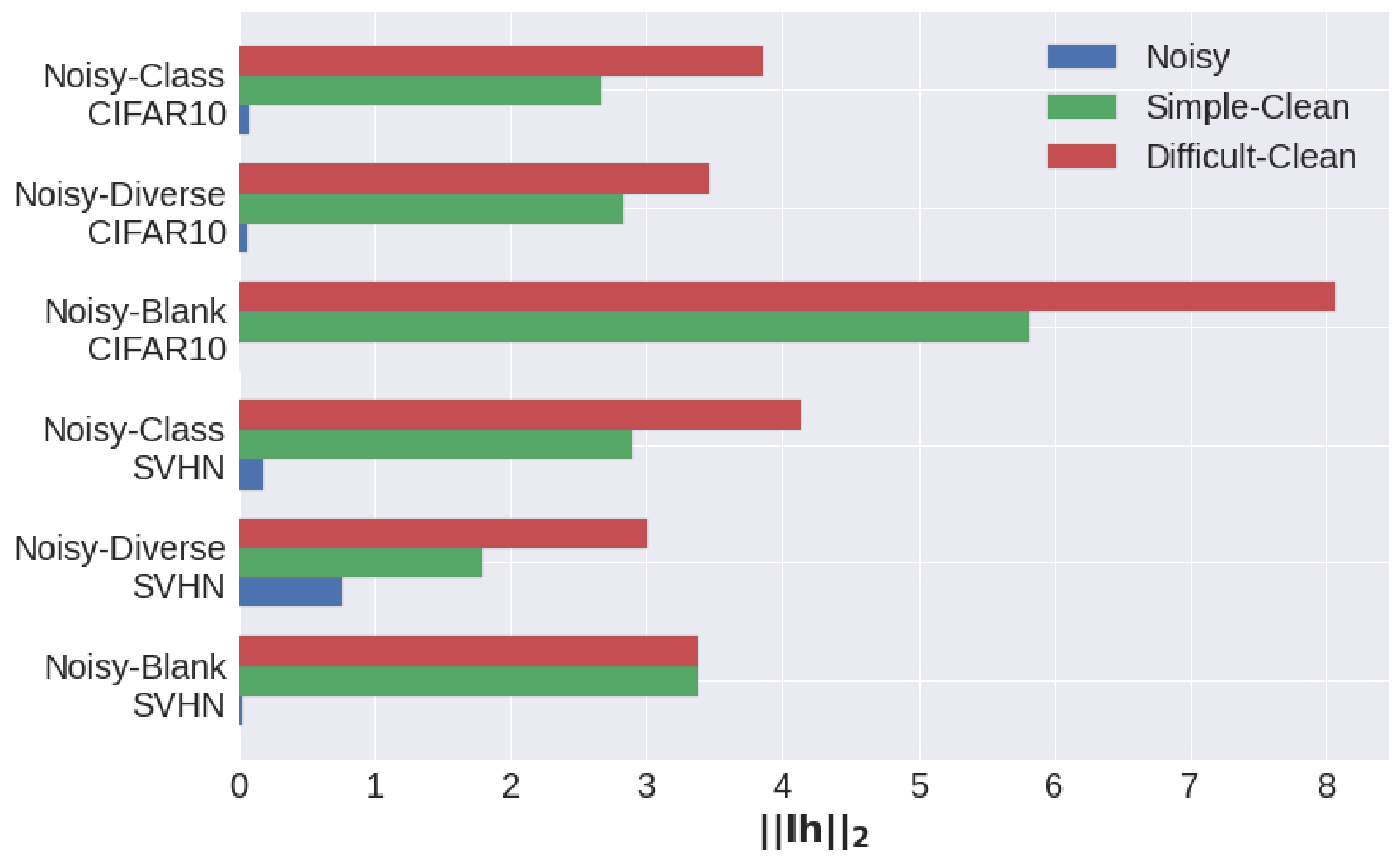}
    \caption{With fine-tuning}
    \label{fig:w-ft}
    \end{subfigure}
    \caption{Average $||\mathbf{lh}||_2$ during first acquisition round across different noising strategies and datasets (higher values indicate examples $\textsc{LHD}$ is more likely to select). $\textsc{LHD}$ has a low preference for selecting noisy examples, despite their high loss values. The experimental setup is same as the one described in Section~\ref{sec:exp} with the ResNet architecture.}
    \label{fig:lh-diff}
\end{figure}

\section{Experiment Details}
\paragraph{Datasets:} 
We experimented with two benchmark datasets - CIFAR10 and SVHN. The CIFAR10 dataset consists of 60000 colored images of size 32 $\times$ 32, split between 50000 training images and 10000 test images. This dataset has ten classes, which include pictures of airplanes, cars, birds, cats, deer, dogs, frogs, horses, ships, and trucks. The SVHN dataset consists of 73257 training samples and 26032 test samples each of size 32 $\times$ 32. Each example is a close-up image of a house number (the ten classes are the digits from 0-9).

We introduce noisy data points belonging to one of the three categories (\emph{Noisy-Blank}, \emph{Noisy-Diverse}, or \emph{Noisy-Class}) to these datasets to get their heteroskedastic counterparts.

\paragraph{Training setup:} All experiments conducted in this paper used Adam optimizer with $\beta_1$=0.9, $\beta_2$=0.999, and a learning rate of $1 \times 10^{-4}$. In each round of active learning, the main model is trained on a batch of 64 examples and for 500 epochs over the training set. The exponentially moving average of the main model is computed using a decay parameter $\alpha$=0.999. After each round of training, a batch of 1000 unlabeled examples is acquired for labeling.

While fine-tuning, we set a probability threshold of 0.8 for selecting high confidence examples from the unlabeled pool. We fine-tune the model for 500 epochs, using the Adam optimizer and learning rate of $1 \times 10^{-2}$. During fine-tuning, we randomly select four augmentations from RandAugment followed by Cutout. 

All the experiments are run using Tesla V100. Table~\ref{tab:time} states the approximate time taken to run ten rounds of active learning for different algorithms using the setup described above. As can be seen, active learning using LHD is faster than BADGE and CORESET, and it gives better or comparable performance to them.

\begin{table}
\centering
\caption{Time taken (in hours) to run 10 rounds of active learning for different algorithms using Tesla V100 in a setup where the noisy examples constitute 80\% of the data points}
\label{tab:time}
\resizebox{0.5\linewidth}{!}{
\begin{tabular}{lccc}

\toprule
\textbf{Method} & \textbf{\shortstack[l]{Time taken w/o fine-tuning}} & \textbf{\shortstack[l]{Time taken w/ fine-tuning}} \\
\midrule
$\text{RAND}$ & 3.5 & 15.5 \\
$\text{CONF}$ & 3.5 & 15.5 \\ 
$\text{MARG}$ & 3.5 & 15.5 \\ 
\midrule
$\text{CORESET}$ & 5.5 & 20.0 \\ 
$\text{BADGE}$ & 6.0 & 21.5 \\ 
$\text{LHD}$ & 4.8 & 16.0 \\ 

\midrule
\bottomrule
\end{tabular}}
\end{table}



\section{Proof of Theorem \ref{thm:1}}
\label{sec:proof}
We first notice that the following proposition from \cite{kawaguchi2020ordered} still holds with the corrupted data with the same proof $g(\Dcal)$:
\begin{proposition} \label{prop:opt_2}
For any $j \in \{1,\dots,n\}$, $\gamma_{j} \le\frac{s}{n}$. 
\end{proposition}
We use this proposition in the following proof of Theorem \ref{thm:1} to bound the effect of replacing one sample in a dataset. 

\begin{proof}[Proof of Theorem \ref{thm:1}]
We find an upper  bound on $\sup_{\theta \in \Theta} \EE_{(x,y)}[\ell(f(x;\theta),y_{})]-L_{q}(\theta;g(\mathcal{D}))$ based on  McDiarmid's inequality.  Define 
$$
\Phi(\Dcal)= \sup_{\theta \in \Theta} \EE_{(x,y)}[\ell(f(x;\theta),y_{})]-L_{q}(\theta;g(\mathcal{D})).
$$ 
Our proof plan is to provide the upper bound on $\Phi(\Dcal)$ by using McDiarmid's inequality. To apply McDiarmid's inequality to $\Phi(\Dcal)$, we first show that $\Phi(\Dcal)$ satisfies the remaining condition of McDiarmid's inequality on the effect of changing one sample. Let $\Dcal$ and $\Dcal'$ be two datasets differing by exactly one point of an arbitrary index $i_{0}$; i.e.,  $\Dcal_i= \Dcal'_i$ for all $i\neq i_{0}$ and $\Dcal_{i_{0}} \neq \Dcal'_{i_{0}}$. Since $(j)$ depends on $g(\Dcal)$, we sometimes write $(j;\Dcal)=(j)$ to stress the dependence on $\Dcal$ under $g$. Then, we provide an upper bound on $\Phi(\Dcal') - \Phi(\Dcal)$ as follows:  

\begin{multline*}
\Phi(\Dcal') - \Phi(\Dcal) \le \sup_{\theta \in \Theta} L_{q}(\theta; g(\mathcal{D}))-L_{q}(\theta;g(\mathcal{D'})).
\\ = \sup_{\theta \in \Theta} \frac{1}{q} \sum_{j=1}^n \gamma_{j} (L_{(j;\Dcal)}(\theta;g(\Dcal))-L_{(j;\Dcal')}(\theta;g(\Dcal')))
\\ \le \sup_{\theta \in \Theta} \frac{1}{q} \sum_{j=1}^n |\gamma_{j}||L_{(j;\Dcal)}(\theta;g(\Dcal))-L_{(j;\Dcal')}(\theta;g(\Dcal'))|
\\ \le\sup_{\theta \in \Theta} \frac{1}{q} \frac{s}{n} \sum_{j=1}^n |L_{(j;\Dcal)}(\theta;g(\Dcal))-L_{(j;\Dcal')}(\theta;g(\Dcal'))|  
\end{multline*}

where the first line follows the property of the supremum, $\sup (a) - \sup (b)\le\sup(a-b)$, the second line follows the definition of $L_{q}$ where $(j;\Dcal) \neq (j; \Dcal')$, and the last line follows Proposition \ref{prop:opt_2} ($|\gamma_{j}| \le\frac{s}{n} $). 

We now bound the last term $ \sum_{j=1}^n |L_{(j;\Dcal)}(\theta;g(\Dcal))-L_{(j;\Dcal')}(\theta;g(\Dcal'))|$. This requires a careful examination because  $|L_{(j;\Dcal)}(\theta;g(\Dcal))-L_{(j;\Dcal')}(\theta;g(\Dcal'))| \neq 0$ for more than one index $j$ (although  $\Dcal$ and $\Dcal'$  differ only by exactly one point). This is because  it is possible to have  $(j;\Dcal ) \neq (j; \Dcal')$ for many indexes $j$ where $(j;\Dcal) $ in $L_{(j;\Dcal)}(\theta;g(\Dcal))$ and $(j;\Dcal') $ in $L_{(j;\Dcal')}(\theta;g(\Dcal'))$.          
To analyze this effect, we now conduct case analysis. Define $l(i_{};\Dcal)$ such that $(j)=i$ where $j=l(i_{};\Dcal)$; i.e., $L_i(\theta;g(\Dcal))=L_{(l(i_{};\Dcal))}(\theta;g(\Dcal))$. 

Consider the case where $l(i_{0};\Dcal') \ge l(i_{0};\Dcal)$. Let $j_1=l(i_{0};\Dcal)$ and $j_2=l(i_{0};\Dcal') $. Then,  

\begin{multline*}
\sum_{j=1}^n |L_{(j)}(\theta;g(\Dcal))-L_{(j)}(\theta;g(\Dcal'))| 
\\ = \sum_{j=j_1}^{j_2-1} |L_{(j)}(\theta;g(\Dcal))-L_{(j)}(\theta;g(\Dcal'))|+|L_{(j_{2})}(\theta;g(\Dcal))-L_{(j_{2})}(\theta;g(\Dcal'))|   
\\ =  \sum_{j=j_1}^{j_2-1} |L_{(j)}(\theta;g(\Dcal))-L_{(j+1)}(\theta;g(\Dcal))|+|L_{(j_{2})}(\theta;g(\Dcal))-L_{(j_{2})}(\theta;g(\Dcal'))| 
\\ =  \sum_{j=j_1}^{j_2-1} (L_{(j)}(\theta;g(\Dcal))-L_{(j+1)}(\theta;g(\Dcal)))+L_{(j_{2})}(\theta;g(\Dcal))-L_{(j_{2})}(\theta;g(\Dcal'))
\\ =  L_{(j_{1})}(\theta;g(\Dcal))-L_{(j_{2})}(\theta;g(\Dcal'))\le M,
\end{multline*}

where the first line uses the fact that $j_2 =l(i_{0};\Dcal')\ge l(i_{0};\Dcal)= j_1$ where $i_{0}$ is the index of samples differing in  $\Dcal$ and $\Dcal'$. The second line follows the equality $(j;\Dcal')=(j+1;\Dcal)$ from $j_1$ to $j_2-1$ in this case. The third line follows the definition of the ordering of the indexes. The fourth line follows the cancellations of the terms from the third line.        

Consider the case where $l(i_{0};\Dcal') < l(i_{0};\Dcal)$. Let $j_1=l(i_{0};\Dcal')$ and $j_2= l(i_{0};\Dcal)$. Then,    
\begin{multline*}
\sum_{j=1}^n |L_{(j)}(\theta;g(\Dcal))-L_{(j)}(\theta;g(\Dcal'))| 
\\ =|L_{(j_{1})}(\theta;g(\Dcal))-L_{(j_{1})}(\theta;g(\Dcal'))|+ \\
\sum_{j=j_1+1}^{j_2} |L_{(j)}(\theta;g(\Dcal))-L_{(j)}(\theta;g(\Dcal'))|   
\\ =  |L_{(j_{1})}(\theta;g(\Dcal))-L_{(j_{1})}(\theta;g(\Dcal'))|+\\
\sum_{j=j_1+1}^{j_2} |L_{(j)}(\theta;g(\Dcal))-L_{(j-1)}(\theta;g(\Dcal))| 
\\ =  L_{(j_{1})}(\theta;g(\Dcal))-L_{(j_{1})}(\theta;g(\Dcal'))+\\
\sum_{j=j_1+1}^{j_2} (L_{(j)}(\theta;g(\Dcal))-L_{(j-1)}(\theta;g(\Dcal)))
\\ =  L_{(j_{1})}(\theta;g(\Dcal'))-L_{(j_{2})}(\theta;g(\Dcal))
\\ \le M.
\end{multline*}
where the first line uses the fact that $j_1 =l(i_{0};\Dcal')< l(i_{0};\Dcal)= j_2$ where $i_{0}$ is the index of samples differing in  $\Dcal$ and $\Dcal'$. The second line follows the equality $(j;\Dcal')=(j-1;\Dcal)$ from $j_1+1$ to $j_2$ in this case. The third line follows the definition of the ordering of the indexes. The fourth line follows the cancellations of the terms from the third line. 

Therefore, in both cases of $l(i_{0};\Dcal') \ge l(i_{0};\Dcal)$ and $l(i_{0};\Dcal') < l(i_{0};\Dcal)$, we have that 
$$
\Phi(\Dcal') - \Phi(\Dcal) \le  \frac{s}{q} \frac{M}{n}.
$$
Similarly,  $\Phi(\Dcal) - \Phi(\Dcal')\le  \frac{s}{q} \frac{M}{n}$, and hence $|\Phi(\Dcal) - \Phi(\Dcal')| \le  \frac{s}{q} \frac{M}{n}$. Thus, by McDiarmid's inequality, for any $\delta>0$, with probability at least $1-\delta$,
$$
\Phi(\Dcal) \le  \EE_{\bar \Dcal}[\Phi(\bar \Dcal)] + \frac{Ms}{q} \sqrt{\frac{\ln(1/\delta)}{2n}}.
$$ 
Moreover, since
\begin{equation}
\sum_{i=1}^n r_i(\theta;g(\Dcal)) L_{i}(\theta;g(\Dcal)) = \sum_{j=1}^n  \gamma_{j}\sum_{i=1}^{n} \mathbbm{1}\{i=(j;\Dcal)\} L_{i}(\theta;g(\Dcal)) = \sum_{j=1}^n  \gamma_{j}L_{(j)}(\theta;g(\Dcal)),
\end{equation}
we have that
$$
L_{q}(\theta;g(\Dcal))=\frac{1}{q} \sum_{i=1}^n r_i(\theta;g(\Dcal)) L_{i}(\theta;g(\Dcal)).
$$ 
Therefore,  

\begin{multline*}
    \EE_{\bar \Dcal}[\Phi(\bar \Dcal)] 
    \\ = \EE_{\bar \Dcal}\left[\sup_{\theta \in \Theta} \EE_{(\bar x',\bar y')}[\ell(f(\bar x';\theta),\bar y')]-L(\theta;  \mathcal{\bar D})+
    L(\theta;\mathcal{\bar D})-L_{q}(\theta; g(\bar \Dcal)) \right] 
    \\ \le \EE_{\bar \Dcal}\left[\sup_{\theta \in \Theta} \EE_{(\bar x',\bar y')}[\ell(f(\bar x';\theta),\bar y')]-L(\theta;\mathcal{\bar D})\right]-\Qcal_{n,q}(\Theta,g) 
    \\ \le\EE_{\bar \Dcal, \bar \Dcal'}\left[\sup_{\theta \in \Theta} \frac{1}{n}\sum_{i=1}^n (\ell(f(\bar x_{i}';\theta), \bar y_{i}')-\ell(f(\bar x_i;\theta),\bar  y_{i}))\right]-\Qcal_{n,q}(\Theta,g)  
    \\ \le\EE_{\xi, \bar \Dcal, \bar \Dcal'}\left[\sup_{\theta \in \Theta} \frac{1}{n}\sum_{i=1}^n  \xi_i(\ell(f(\bar x_{i}';\theta), \bar y_{i}')-\ell(f(\bar x_i;\theta),\bar y_{i}))\right]-\Qcal_{n,q}(\Theta,g) 
    \\ \le 2\Rfra_{n}(\Theta) - \Qcal_{n,q}(\Theta,g). 
\end{multline*}

where the third line and the last line follow the subadditivity of supremum, the forth line follows the Jensen's inequality and the convexity of  the 
supremum, the fifth line follows that for each $\xi_i \in \{-1,+1\}$, the distribution of each term $\xi_i (\ell(f(\bar x_{i}';\theta),\bar y_{i}')-\ell(f(\bar x_i;\theta),\bar y_{i}))$ is the  distribution of  $(\ell(f(\bar x_{i}';\theta),\bar y_{i}')-\ell(f(\bar x_i;\theta),\bar y_{i}))$  since $\bar \Dcal$ and $\bar \Dcal'$ are drawn iid with the same distribution. Therefore, for any $\delta>0$, with probability at least $1-\delta$,

$$
\Phi(\Dcal) \le2\Rfra_{n}(\Theta)   -  \Qcal_{n,q}(\Theta,g)+ \frac{Ms}{q} \sqrt{\frac{\ln(1/\delta)}{2n}}.
$$ 
Finally, since changing one data point in $\Dcal$ changes  $\hat \Rfra_{n}(\Theta)$ by at most $M/m$, McDiarmid's inequality implies that
for any $\delta>0$, with probability at least $1-\delta$,
$$
\Rfra_{n}(\Theta)    \le \hat \Rfra_{n}(\Theta)    +M \sqrt{\frac{\ln (1/\delta)}{2n}}.
$$ 
By taking union bound, we obtain the statement of this theorem.

\end{proof}

\end{document}